\documentclass{article} 
\usepackage{iclr,times}

\usepackage{amsmath,amsfonts,bm}

\def\eqref#1{equation~\ref{#1}}

\def\1{\bm{1}}

\DeclareMathAlphabet{\mathsfit}{\encodingdefault}{\sfdefault}{m}{sl}
\SetMathAlphabet{\mathsfit}{bold}{\encodingdefault}{\sfdefault}{bx}{n}

\usepackage{hyperref}
\usepackage{url}

\title{MotionSpaceFlow: Representation-Aware Flow Matching in Direct Motion Space}

\author{
    Qing Yu \quad Kent Fujiwara\vspace{.2em}\\
    \hspace{3pt}LY Corporation, Tokyo, Japan \vspace{.2em}\\
    \hspace{3pt}\texttt{\{yu.qing, kent.fujiwara\}@lycorp.co.jp}
}

\usepackage{pifont}
\usepackage{booktabs}
\usepackage{makecell}
\usepackage{multirow}
\usepackage{rotating}
\usepackage{algorithm}
\usepackage{algpseudocode}
\usepackage{capt-of}
\usepackage{url}
\usepackage{mathtools}
\usepackage{amsthm}
\usepackage{subcaption}
\usepackage{colortbl}
\usepackage{hhline}
\usepackage{nicematrix}
\usepackage{enumitem}
\usepackage{placeins}
\usepackage{graphicx}
\usepackage{xspace}
\usepackage{tikz}
\usetikzlibrary{arrows.meta,positioning,fit}

\newcommand{\Tref}[1]{Table~\ref{#1}}
\newcommand{\Eref}[1]{Eq.~(\ref{#1})}
\newcommand{\Fref}[1]{Fig.~\ref{#1}}

\newcommand{\Sref}[1]{Section~\ref{#1}}

\newcommand{\eg}{\textit{e.g.}\xspace}

\newcommand{\cmark}{\ding{51}}%
\newcommand{\xmark}{\ding{55}}%

\newcommand{\x}{\mathbf{x}}
\newcommand{\y}{\mathbf{y}}

\newcommand{\method}{\textsc{MSFlow}\xspace}

\newtheorem{proposition}{Proposition}

\iclrfinalcopy 
\begin{document}

\maketitle

\begin{abstract}
Recent advances in diffusion and flow models have substantially improved text-driven human motion generation. Yet most methods generate in low-dimensional, temporally downsampled latent spaces learned primarily for reconstruction, a bottleneck that can limit generation quality and preclude direct manipulation of individual frames and joints. We introduce MotionSpaceFlow (\method), a representation-aware flow-matching framework that predicts clean motion directly in continuous motion space without a learned encoder or decoder. To account for the anisotropic structure of direct motion representations, we propose representation-aware noise scaling and show how the initial Gaussian source scale governs the covariance of intermediate probability-path marginals. We further introduce a Representation-Aware Multimodal Diffusion Transformer (RA-MMDiT), which jointly updates token-level language and full-resolution motion features through joint attention while adapting temporal information flow to the motion representation: causal attention for incremental features defined by frame-to-frame changes, and bidirectional attention for global features such as absolute joint coordinates. Across different datasets and motion representations, \method achieves state-of-the-art text-to-motion performance. Its global representation variant additionally enables zero-shot, inference-time control over any joint or frame through projection sampling without control-conditioned training, delivering leading motion quality with exact constraint satisfaction.
\end{abstract}

\section{Introduction}
\label{sec:intro}

Synthesizing human motion from natural language requires translating a high-level description into a coherent physical trajectory. Diffusion and flow models~\citep{ho2020denoising,song2021score,rombach2022high,lipman2022flow} address this one-to-many problem by transforming Gaussian noise into complete motion sequences. Most recent methods generate in low-dimensional, temporally downsampled spaces produced by variational autoencoders (VAEs)~\citep{kingma2013auto,chen2023executing,dai2024motionlcm,hong2025salad}, reducing sequence length and generation cost.
Others discretize motion with a VQ-VAE~\citep{van2017neural} and apply transformer-based generative models to the resulting token sequence~\citep{guo2022tm2t,jiang2023motiongpt,guo2024momask}. Although effective, both continuous and discrete latent approaches learn representations primarily through reconstruction. Generation quality can therefore be limited by autoencoder capacity and reconstruction fidelity~\citep{yao2025reconstruction,zheng2026diffusion}, while latent tokens lose direct correspondence to individual frames and joints, complicating fine-grained spatial control.

\begin{figure*}[t]
    \centering
    \includegraphics[width=\textwidth]{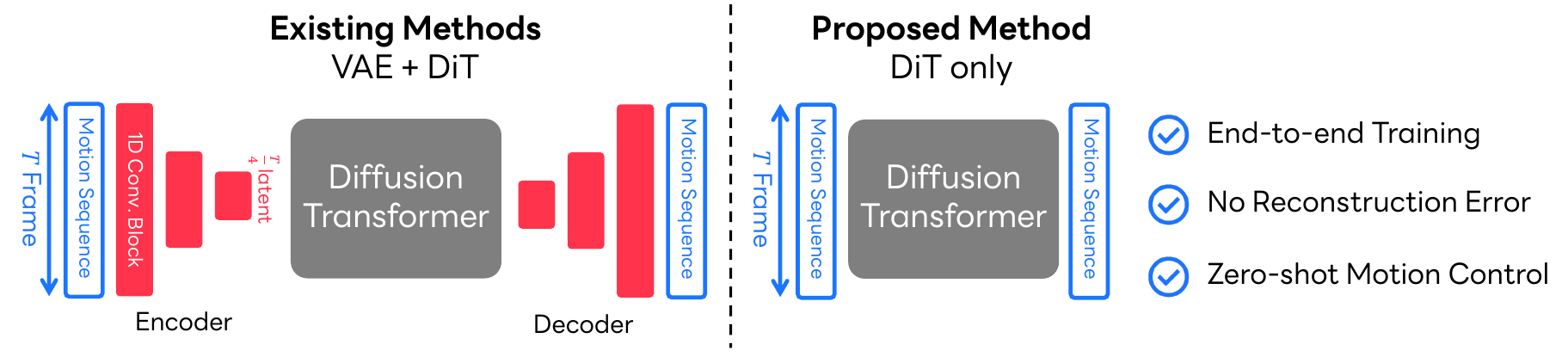}
    \caption{\textbf{Latent-space versus direct motion-space generation.} Existing methods encode a $T$-frame sequence into a shorter latent sequence, apply a diffusion transformer in the latent space, and decode the result back to motion. \method removes the motion encoder and decoder and applies a flow-matching transformer directly to the full sequence. This end-to-end design avoids autoencoder reconstruction error and preserves direct access to frames and joints for zero-shot spatial control.
    }
    \label{fig:overview}
\end{figure*}

As illustrated in~\Fref{fig:overview}, we instead generate full-resolution motion variables with a Diffusion Transformer (DiT)-only model. Removing the motion encoder and decoder eliminates the reconstruction bottleneck and preserves the frame--joint structure required for inference-time spatial control. Direct-space generation, however, is not achieved by simply removing the autoencoder. Coordinate-wise z-normalization does not eliminate correlations across joints, feature types, and frames because motion representations combine continuous 3D quantities (\eg, root-relative joint positions), continuous 6D rotations (\eg, joint orientations), and categorical variables (\eg, foot-contact labels). The flow must therefore traverse an anisotropic data distribution, adapt temporal dependencies to the represented variables, and align a short prompt with full-resolution motion.

To address these challenges, we introduce MotionSpaceFlow (\method), a representation-aware framework for flow matching in continuous motion space. Following JiT~\citep{li2026back}, \method predicts the clean motion endpoint and analytically converts it into a velocity field for deterministic ODE sampling. Representation-aware noise scaling incorporates the Gaussian source scale into the probability path used for both training and sampling, treating it as a parameter of the probability path rather than merely as a control for inference-time randomness. Our analysis shows how this scale controls signal emergence and the anisotropy of intermediate path marginals.

Our Representation-Aware Multimodal DiT (RA-MMDiT) jointly updates token-level language and frame-level motion features to support fine-grained alignment across the full sequence. Crucially, its temporal attention graph is configured to match the underlying representation. The standard 263-dimensional HumanML3D representation~\citep{guo2022generating} includes root velocities whose accumulation recovers the global trajectory of the whole sequence. RA-MMDiT uses causal attention for these incremental features: all frames are updated in parallel at each ODE step, while each motion token attends only to its prefix and the complete text condition. For globally coupled absolute XYZ coordinates of all joints, it instead uses bidirectional attention, providing two-sided context and propagating spatial constraints before and after a controlled keyframe. This representation-aware attention lets RA-MMDiT handle both incremental and absolute motion representations.

On HumanML3D, the causal 263D variant achieves state-of-the-art text-to-motion performance, and our results reveal a clear contrast between the optimal attention designs for incremental 263D and absolute XYZ representations.
The XYZ variant further enables any-joint, any-frame projection sampling without control-conditioned training, delivering leading motion quality and exact constraint satisfaction under the OmniControl~\citep{xie2024omnicontrol} protocol. We also investigate the broader applicability of \method on SnapMoGen dataset~\citep{guo2025snapmogen}. Our main contributions are:

\begin{samepage}
    \begin{itemize}[leftmargin=*,nosep]
        \item \textbf{Direct motion-space flow matching.} We introduce \method, which generates continuous motion through clean-endpoint prediction and ODE sampling without a learned motion encoder or decoder, and achieves state-of-the-art text-to-motion results.
        \item \textbf{Representation-aware noise scaling.} We characterize how Gaussian source scale controls signal emergence and the covariance conditioning of intermediate direct-space path marginals.
        \item \textbf{Representation-aware temporal modeling.}  We introduce RA-MMDiT, which jointly updates token-level language and full-resolution motion features, and show that causal attention favors incremental representations whereas bidirectional attention is crucial for absolute coordinates.
        \item \textbf{Training-free spatial control.} We derive projection sampling for any-joint, any-frame constraints using a model trained only for text-to-motion generation, with exact constraint satisfaction.
    \end{itemize}
\end{samepage}

\section{Related Work}
\label{sec:related}

\subsection{Motion Representations and Temporal Modeling for Motion Generation}
\enlargethispage{\baselineskip}
Early neural approaches generate motion with recurrent or convolutional architectures~\citep{yan2019convolutional,zhao2020bayesian,guo2020action2motion,ghosh2021synthesis}. Recent text-conditioned models operate on discrete tokens~\citep{guo2022tm2t,zhang2023t2m,guo2024momask,pinyoanuntapong2024mmm,zou2024parco}, continuous motion features~\citep{mdm2022human,zhang2022motiondiffuse,dabral2023mofusion}, or autoencoded latent sequences~\citep{petrovich2022temos,chen2023executing,dai2024motionlcm,hong2025salad}. Discrete and autoencoded representations shorten the modeled sequence, but temporal downsampling limits information capacity and decoded fidelity remains bounded by autoencoder reconstruction quality. \method instead learns transport in continuous motion space, preserving frame- and joint-level access.

Autoregressive models enforce temporal causality by predicting future tokens or segments from past context. Discrete autoregressive methods~\citep{zhang2023t2m,jiang2023motiongpt} can suffer from exposure bias~\citep{schmidt2019generalization}, whereas diffusion-based variants~\citep{meng2024rethinking,zhao2024dartcontrol,xiao2025motionstreamer,yu2026causal} combine iterative refinement with causal backbones. Bidirectional denoisers instead use complete temporal context. We compare matched causal and bidirectional variants to test whether the preferred dependency graph depends on the representation.

\subsection{Flow Matching and Spatial Control}
Flow matching~\citep{albergo2022building,albergo2023stochastic,lipman2022flow} learn vector fields that transport a simple source distribution to data, and scalable interpolant and Diffusion Transformers~\citep{peebles2023scalable,ma2024sit} provide effective parameterizations of such fields. We adopt clean-endpoint prediction from JiT~\citep{li2026back}, which allows the method to operate effectively in high-dimensional spaces, and study its interaction with Gaussian source scale and temporal attention in anisotropic motion representations. Some prior work selects or optimizes the initial noise of a fixed pretrained diffusion process to control generation~\citep{mao2024lottery,wang2024silent,zhou2024golden,harrington2026noisediv,ota2026winro}.

In motion generation, prior methods impose spatial constraints or edits through control-conditioned training, guidance, test-time optimization, or large-scale controllable modeling~\citep{karunratanakul2023guided,athanasiou2024motionfix,huang2024controllable,petrovich2024multi,xie2024omnicontrol,liu2024programmable,rempe2026kimodo}. ProjFlow~\citep{watanabe2026projflow} instead projects clean endpoint estimates during sampling without control-conditioned training. \method applies endpoint projection directly to XYZ motion, enabling arbitrary frame--joint--axis constraints.

\subsection{Motion--Language Conditioning}
Unlike motion--language alignment models~\citep{tevet2022motionclip,petrovich2023tmr,yu2024exploring,fujiwara2024chronologically,yu2025remogpt}, which jointly update motion and text features through contrastive learning, text-to-motion generators commonly condition on a frozen language encoder. In direct motion-space generation, preserving textual semantics is significantly more challenging because no temporally downsampled latent representation is available to bridge the modality gap. \method therefore refines token-level DistilBERT features~\citep{sanh2019distilbert} as a function of flow time and couples them with motion tokens through multimodal attention, without an auxiliary motion--language alignment objective.

\section{Method}
\label{sec:method}
\method generates motion directly in the data representation, without a learned motion encoder or decoder. \Fref{fig:framework} shows that the method consists of two components. First, a clean motion sequence, represented using either incremental 263D features or global XYZ coordinates, is interpolated with a scaled Gaussian source to form a noised motion state. Second, a Representation-Aware Multimodal Diffusion Transformer (RA-MMDiT) combines this state with token-level text features to predict the clean motion endpoint. The representation informs both the source scale and the temporal attention graph: causal for incremental features and bidirectional for absolute coordinates. The predicted endpoint is converted into an ODE velocity for the training objective and inference-time sampling. For XYZ motion, the endpoint can also be projected to satisfy the spatial constraints during inference.

\begin{figure*}[t]
    \centering
    \includegraphics[width=\textwidth]{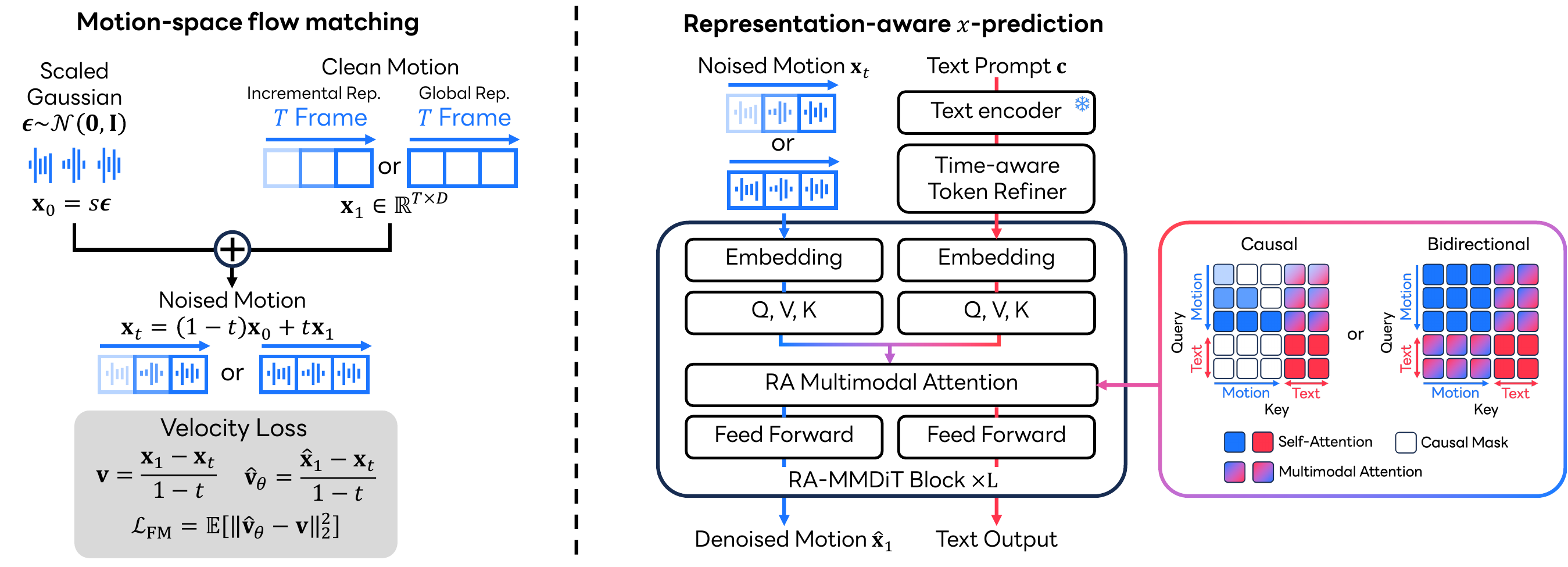}
    \caption{\textbf{Framework of \method.} \textbf{Left:} motion-space flow matching interpolates a scaled Gaussian source $\x_0=s\boldsymbol\epsilon$ with a clean $T$-frame motion $\x_1$ in either an incremental or global representation. \textbf{Right:} RA-MMDiT jointly processes the noised motion $\x_t$ and time-refined text tokens to predict the clean endpoint $\hat\x_1$. Its attention mask follows the motion representation: a motion query attends only to its prefix and all text tokens in the causal variant, whereas the bidirectional variant uses the complete motion--text context.}
    \label{fig:framework}
\end{figure*}

\subsection{Problem Setting and Direct Motion Representations}
\label{sec:direct-space}

Let $\x_1\in\mathbb R^{T\times D}$ be a normalized motion sequence, where $T$ is the number of frames and $D$ is the feature dimension, and let $\mathbf c$ denote a text prompt, such as a natural language description of the motion. We model the sequence as a continuous tensor at its original temporal resolution. We consider two commonly used representations with different temporal semantics:
\begin{itemize}[leftmargin=*,nosep]
    \item \textbf{Incremental 263D motion.} The standard HumanML3D representation~\citep{guo2022generating} contains root yaw and planar velocities, root height, root-relative joint positions, 6D joint rotations, local joint velocities, and foot-contact channels. In particular, the global root trajectory is obtained by accumulating frame-to-frame velocities.
    \item \textbf{Global XYZ motion.} Each frame contains the absolute 3D coordinates of 22 joints, resulting in $D=66$. These variables are globally coupled across time but expose every frame and joint directly, which makes them suitable for spatial control.
\end{itemize}

In both cases, the flow acts on $\x_1$ itself; there is no learned map between a temporally compressed latent sequence and motion. The following bound isolates the reconstruction constraint that this design removes. Let $P_1$ be the data distribution over vectorized motions in $\mathbb R^{TD}$, and let $g:\mathcal Z\rightarrow\mathbb R^{TD}$ be a fixed decoder with range $\mathcal R_g=\{g(\mathbf z):\mathbf z\in\mathcal Z\}$.

\begin{proposition}[Autoencoder fidelity floor]
\label{prop:decoder-floor}
For any generated distribution $Q$ supported on $\mathcal R_g$,
\begin{equation}
    W_2^2(P_1,Q)
    \geq
    \mathbb E_{\mathbf X\sim P_1}
    \left[\operatorname{dist}^2(\mathbf X,\mathcal R_g)\right],
    \label{eq:decoder-floor}
\end{equation}
\end{proposition}
where $W_2^2$ denotes the squared 2-Wasserstein distance. Every decoded sample must lie in the decoder range, regardless of the latent generator's expressiveness. Because the lower bound can be zero when that range contains the data support, it does not claim that direct generation is universally better. Rather, it shows that direct motion-space flow eliminates this particular decoder-induced fidelity floor while retaining ordinary generative error.

\subsection{Clean-Motion Prediction and Flow Sampling}
\label{sec:x-prediction}

The left side of~\Fref{fig:framework} defines the probability path. During training, we sample one time $t\sim\mathcal U(0,1)$ per sequence and one Gaussian tensor $\boldsymbol\epsilon$ of the same shape as the motion, then form
\begin{equation}
    \x_0=s\boldsymbol\epsilon,
    \qquad
    \boldsymbol\epsilon\sim\mathcal N(\mathbf0,\mathbf I),
    \qquad
    \x_t=(1-t)\x_0+t\x_1,
    \label{eq:linear-path}
\end{equation}
where $s>0$ is the source scale. The same $t$ is used for all valid frames so that $\x_t$ remains a coherent sequence. Rather than predicting velocity directly, RA-MMDiT predicts the clean endpoint,
\begin{equation}
    \hat\x_1=f_\theta(\x_t,t,\mathbf c).
    \label{eq:x-prediction}
\end{equation}
This $x$-prediction parameterization follows JiT~\citep{li2026back} and provides the network a target with the same physical meaning at every noise level, where $\hat\x_1$ always retains the semantics and units of the original motion representation. For the linear path, the exact sample-wise velocity and its endpoint-based estimate are calculated as follows:
\begin{equation}
    \mathbf v=\x_1-\x_0=\frac{\x_1-\x_t}{1-t},
    \qquad
    \hat{\mathbf v}_\theta=\frac{\hat\x_1-\x_t}{1-t}.
    \label{eq:jit-velocity}
\end{equation}
Thus, clean-motion prediction remains compatible with the flow ODE $d\x_t/dt=\hat{\mathbf v}_\theta$. The denominator in~\Eref{eq:jit-velocity} becomes numerically unstable near $t=1$. We therefore clip it as follows:
\begin{equation}
    d_t=\max(1-t,\sigma_{\min}),\qquad
    \hat{\mathbf v}_\theta=\frac{\hat\x_1-\x_t}{d_t},\qquad
    \mathbf v^*=\frac{\x_1-\x_t}{d_t},
    \label{eq:x-to-v}
\end{equation}
and optimize the valid-frame residual loss as follows:
\begin{equation}
    \mathcal L_{\mathrm{FM}}
    =\mathbb E\left[
    \frac{1}{|\Omega|}\sum_{i\in\Omega}
    \left\|\hat{\mathbf v}_{\theta,i}-\mathbf v_i^*\right\|_2^2
    \right],
    \label{eq:jit-loss}
\end{equation}
where $\Omega$ is the set of non-padding frames. For $t\leq1-\sigma_{\min}$, this is exactly the linear-path velocity objective, equivalently weighting endpoint error by $(1-t)^{-2}$. The terminal interval uses a bounded surrogate. At inference, we initialize the entire sequence from $\mathcal N(\mathbf0,s^2\mathbf I)$ and integrate all frames jointly with a fixed-step Heun solver~\citep{Heun1900} and a final Euler update.

\subsection{Representation-Aware Noise Scaling}
\label{sec:noise-scaling}

Following standard practice~\citep{guo2022generating,meng2025absolute}, we normalize every motion channel before training. However, normalization alone does not make the motion distribution spherical. It rescales individual channels without removing correlations across joints, feature types, or frames. The resulting 263D and XYZ distributions can therefore have much more variation in some directions than in others, whereas the Gaussian source varies equally in every direction. We verify this structure empirically in Appendix~\ref{app:motion-correlation}.

The source scale $s$ in~\Eref{eq:linear-path} is the standard deviation of the initial Gaussian noise. It is fixed during both training and sampling and should not be confused with an inference-time temperature. Increasing $s$ has two intuitive effects: stronger noise causes the motion signal to appear later along the path, while its isotropic covariance makes intermediate states better conditioned.

To formalize these effects, let $\mathbf X_1$ be a vectorized clean motion for a fixed-length sequence. Assume $\mathbb E[\mathbf X_1]=\mathbf0$ and denote its covariance by $\boldsymbol\Sigma$, with smallest and largest eigenvalues $\lambda_{\min}$ and $\lambda_{\max}$.

\begin{proposition}[Effects of source scale on the linear flow path]
\label{prop:source-geometry}
Let $\mathbf X_0\sim\mathcal N(\mathbf0,s^2\mathbf I)$ be independent of $\mathbf X_1$, and let
\begin{equation}
    \mathbf X_t=t\mathbf X_1+(1-t)\mathbf X_0.
    \label{eq:source-geometry-path}
\end{equation}
Then $s$ controls when the motion signal emerges and how well conditioned the intermediate distribution is:
\begin{enumerate}[leftmargin=*,label=(\roman*),nosep]
    \item For any unit eigenvector $\mathbf u$ of $\boldsymbol\Sigma$ with eigenvalue $\lambda_{\mathbf u}$,
    \begin{equation}
        \operatorname{SNR}_{\mathbf u}(t;s)
        =\frac{t^2\lambda_{\mathbf u}}{(1-t)^2s^2},
        \qquad
        t_{\mathbf u}^*=\frac{s}{s+\sqrt{\lambda_{\mathbf u}}},
        \label{eq:path-snr}
    \end{equation}
    where $t_{\mathbf u}^*$ is the time at which the signal and noise variances are equal. A larger $s$ moves this time closer to the clean endpoint $t=1$, so the motion signal emerges later.
    \item The intermediate covariance and its condition number for $t<1$ are
    \begin{align}
        \operatorname{Cov}(\mathbf X_t)
        &=t^2\boldsymbol\Sigma+(1-t)^2s^2\mathbf I,
        \label{eq:path-covariance}\\
        \kappa_t(s)
        &=\frac{t^2\lambda_{\max}+(1-t)^2s^2}
        {t^2\lambda_{\min}+(1-t)^2s^2}.
        \label{eq:path-condition}
    \end{align}
    The condition number $\kappa_t(s)$ is non-increasing in $s$. Thus, a larger isotropic source scale improves conditioning by reducing the relative disparity between high- and low-variance directions, while leaving their absolute eigenvalue gap unchanged.
\end{enumerate}
\end{proposition}

Our ablations show that $s=5$ works well for both the 263D and XYZ representations. We therefore use this value in our primary models.

\subsection{Representation-Aware Multimodal Transformer}
\label{sec:architecture}

The right side of~\Fref{fig:framework} shows RA-MMDiT. Each motion frame is projected to a 512-dimensional token. A frozen DistilBERT encoder~\citep{sanh2019distilbert} produces 768-dimensional token-level language features. We use a two-layer, time-aware text module, which we call the Token Refiner, to project these features to the model width and adapt them to the current flow time. This makes the otherwise static language features responsive to the changing motion signal along the flow path. Rotary position embeddings~\citep{su2024roformer} encode motion order, adaptive layer normalization~\citep{peebles2023scalable} injects flow time, and the text condition is independently replaced by a null embedding with probability $0.1$ to train classifier-free guidance.

The network contains eight RA-MMDiT blocks with four attention heads~\citep{vaswani2017attention}. Each block has modality-specific query, key, value, and feed-forward projections, while joint attention concatenates motion and text keys and values. This allows every motion frame to retrieve relevant words without first compressing the prompt to a single vector. The attention mask is the only architectural difference between the two representation variants:
\begin{itemize}[leftmargin=*,nosep]
    \item In the \textbf{causal} variant, motion query $i$ attends to motion frames $1{:}i$ and all valid text tokens. Text queries attend only to text tokens, preventing future motion from reaching an earlier frame indirectly through the updated language stream.
    \item In the \textbf{bidirectional} variant, motion and text queries attend to the complete valid motion--text sequence. Padding tokens are masked in both variants.
\end{itemize}

The mask reflects each representation's structure. Latent models can handle the directionality of incremental 263D features implicitly through their encoder and decoder, whereas direct motion-space generation must model it explicitly. Because 263D root translation and yaw are integrated from frame-to-frame velocities, we use causal masking to encourage forward-consistent dependencies and limit reliance on bidirectional smoothing.

Absolute XYZ trajectories are globally coupled, so bidirectional context can propagate temporal and skeletal constraints. Causal attention still updates all frames jointly at each ODE step, but prevents future perturbations from affecting earlier frames. These filtering and smoothing properties motivate the causal 263D and bidirectional XYZ variants evaluated in~\Sref{sec:attention-study}; formal analyses appear in Appendix~\ref{app:attention-representation}. At sampling time, classifier-free guidance~\citep{ho2022classifier} is applied to the predicted endpoints and the result is converted to a velocity using~\Eref{eq:x-to-v}. Because the model predicts the target motion variables directly, the integrated endpoint requires no motion decoder.

\subsection{Inference-Time Any-Joint, Any-Frame Control}
\label{sec:zero-shot-control}

Direct XYZ endpoints enable zero-shot spatial control at inference. The bidirectional XYZ model is trained on paired motions and text, without constraint masks or target coordinates. At inference, let $\mathbf M\in\{0,1\}^{T\times66}$ select any valid frame--joint--axis entries and $\y$ contain their targets in normalized XYZ coordinates. At each solver step, we first project the predicted clean endpoint:
\begin{equation}
    \hat\x_1^{\mathrm{proj}}
    =(\mathbf1-\mathbf M)\odot\hat\x_1
    +\mathbf M\odot\y.
    \label{eq:control-projection}
\end{equation}

Simply overwriting the same entries of the noised state $\x_t$ with clean coordinates would mix two noise levels. Instead, we write the linear path as $\x_t=\alpha_t\x_1+\sigma_t\x_0$, where $\alpha_t=t$ and $\sigma_t=1-t$, and recover the current source estimate $\hat\x_0=(\x_t-\alpha_t\hat\x_1)/\sigma_t$ for $t<1$. For the next solver time $t'$, the projection sampler refreshes that estimate with a new source-scale noise $\epsilon_s$ and recomposes a path-consistent state:
\begin{align}
    \eta_{t'}&=\operatorname{clip}\!\left(
    \frac{1-\sigma_{t'}}{s^2},0,1\right),\\
    \widetilde\x_0&=\sqrt{1-\eta_{t'}}\,\hat\x_0
    +\sqrt{\eta_{t'}}\,\boldsymbol\epsilon_s,
    \quad \boldsymbol\epsilon_s\sim\mathcal N(\mathbf0,s^2\mathbf I),
    \label{eq:control-refresh}\\
    \x_{t'}&=\alpha_{t'}\hat\x_1^{\mathrm{proj}}
    +\sigma_{t'}\widetilde\x_0.
    \label{eq:control-recompose}
\end{align}
The injected noise uses the source scale of the training path, and a final hard projection removes residual numerical error at the endpoint. Following \mbox{ProjFlow}~\citep{watanabe2026projflow}, our projection sampler also adopts its kinematics-aware metric and pseudo-observation construction. We refer readers to that paper for details. Neither $\mathbf M$ nor $\y$ appears in the training objective, so ``training-free control'' refers specifically to introducing spatial constraints only during inference.

\section{Experiments}
\label{sec:experiments}

\subsection{Experimental Setup}

\paragraph{Dataset and representations.}
We conduct experiments on HumanML3D~\citep{guo2022generating}, using sequences of at most 192 frames at 20 frames per second. The primary model generates normalized 263-dimensional HumanML3D features. The XYZ variant generates shared-axis-normalized positions of 22 joints, flattened to 66 values per frame. Following previous methods~\citep{meng2024rethinking,meng2025absolute,yu2026causal}, both variants are evaluated with the standard 67-dimensional root and rotation-invariant joint-position evaluator introduced in~\citep{meng2024rethinking}. XYZ outputs are converted to this space through the HumanML3D post-processing pipeline. To test broader applicability, we additionally train \method on the 272D HumanML3D representation used by MotionStreamer~\citep{xiao2025motionstreamer} and on the SnapMoGen dataset using its native 296D representation~\citep{guo2025snapmogen}. Full protocols and results appear in Appendix~\ref{app:272d-results} and~\ref{app:snap-results}, respectively.

\paragraph{Evaluation metrics.}
Following prior works, we report motion-embedding Fr\'{e}chet distance (FID), R-Precision at Top 1--3, matching distance (MM-Dist), multimodality (MModality), and motion--language cosine similarity (CLIP score). Each generation result averages 10 stochastic evaluations.

\paragraph{Implementation details.}
The 68M-parameter transformer uses eight RA-MMDiT blocks, source scale $s=5$, and 50-step flow sampling. Full training and inference details appear in Appendix~\ref{app:repro}.

\subsection{Quantitative Text-to-Motion Results}
\label{sec:main-results}

\begin{table}[t]
\centering
\caption{\textbf{Results of text-to-motion performance on HumanML3D.} We use the 67D evaluator following MARDM. \method (263D) denotes the incremental-263D model, while \method (XYZ) denotes the global-XYZ variant. The average is reported over 10 runs with 95\% confidence intervals. \textbf{Bold} indicates the best result, and \underline{underline} denotes the second-best result.}
\label{tab:main-results}
\setlength{\tabcolsep}{2.7pt}
\resizebox{\linewidth}{!}{
\begin{tabular}{l|c|ccc|cccc}
\toprule
\multirow{2}{*}{Method} & \multirow{2}{*}{Representation} & \multicolumn{3}{c|}{R-Precision$\uparrow$} & \multirow{2}{*}{FID$\downarrow$} & \multirow{2}{*}{MM-Dist$\downarrow$} & \multirow{2}{*}{MModality$\uparrow$} & \multirow{2}{*}{CLIP$\uparrow$} \\
\cline{3-5}
~ & ~ & Top 1 & Top 2 & Top 3 & & & & \\
\midrule
T2M-GPT~\citep{zhang2023t2m} & \multirow{3}{*}{\shortstack{Discrete\\latent}} & $.470^{\pm.003}$ & $.659^{\pm.002}$ & $.758^{\pm.002}$ & $.335^{\pm.003}$ & $3.505^{\pm.017}$ & $2.018^{\pm.053}$ & $.607^{\pm.005}$ \\
MMM~\citep{pinyoanuntapong2024mmm} & & $.487^{\pm.003}$ & $.683^{\pm.002}$ & $.782^{\pm.002}$ & $.132^{\pm.004}$ & $3.359^{\pm.019}$ & $2.241^{\pm.073}$ & $.635^{\pm.003}$ \\
MoMask~\citep{guo2024momask} & & $.490^{\pm.004}$ & $.687^{\pm.003}$ & $.786^{\pm.003}$ & $.116^{\pm.006}$ & $3.353^{\pm.010}$ & $1.263^{\pm.079}$ & $.637^{\pm.003}$ \\
\midrule
MLD~\citep{chen2023executing} & \multirow{5}{*}{\shortstack{Continuous\\latent}} & $.461^{\pm.004}$ & $.651^{\pm.004}$ & $.750^{\pm.003}$ & $.431^{\pm.014}$ & $3.445^{\pm.019}$ & $\underline{3.506^{\pm.031}}$ & $.615^{\pm.003}$ \\
SALAD~\citep{hong2025salad} & & $.552^{\pm.003}$ & $.748^{\pm.003}$ & $.839^{\pm.002}$ & $.124^{\pm.005}$ & $2.990^{\pm.010}$ & $1.833^{\pm.081}$ & $.671^{\pm.001}$ \\
MARDM~\citep{meng2024rethinking} & & $.500^{\pm.004}$ & $.695^{\pm.003}$ & $.795^{\pm.003}$ & $.114^{\pm.007}$ & $3.270^{\pm.009}$ & $2.231^{\pm.071}$ & $.642^{\pm.002}$ \\
ACMDM~\citep{meng2025absolute} & & $.522^{\pm.002}$ & $.713^{\pm.002}$ & $.807^{\pm.002}$ & $.058^{\pm.004}$ & $3.205^{\pm.008}$ & $2.077^{\pm.083}$ & $.652^{\pm.001}$ \\
CMDM~\citep{yu2026causal} & & $.563^{\pm.004}$ & $\underline{.759^{\pm.003}}$ & $\underline{.849^{\pm.002}}$ & $.078^{\pm.003}$ & $2.920^{\pm.007}$ & $1.827^{\pm.094}$ & $\underline{.685^{\pm.001}}$ \\
\midrule
MDM~\citep{mdm2022human} & & $.440^{\pm.007}$ & $.636^{\pm.006}$ & $.742^{\pm.004}$ & $.518^{\pm.032}$ & $3.640^{\pm.028}$ & $\mathbf{3.604^{\pm.031}}$ & $.578^{\pm.003}$ \\
MotionDiffuse~\citep{zhang2022motiondiffuse} & & $.450^{\pm.006}$ & $.641^{\pm.005}$ & $.753^{\pm.005}$ & $.778^{\pm.035}$ & $3.490^{\pm.023}$ & $3.179^{\pm.046}$ & $.606^{\pm.004}$ \\
ReMoDiffuse~\citep{zhang2023remodiffuse} & & $.468^{\pm.003}$ & $.653^{\pm.003}$ & $.754^{\pm.005}$ & $.883^{\pm.021}$ & $3.414^{\pm.020}$ & $2.703^{\pm.154}$ & $.621^{\pm.003}$ \\
\rowcolor[gray]{0.90}\textbf{\method (263D)} & & $\mathbf{.571^{\pm.005}}$ & $\mathbf{.764^{\pm.003}}$ & $\mathbf{.853^{\pm.003}}$ & $\underline{.046^{\pm.004}}$ & $\mathbf{2.890^{\pm.010}}$ & $1.450^{\pm.079}$ & $\mathbf{.686^{\pm.000}}$ \\
\rowcolor[gray]{0.90}\textbf{\method (XYZ)} & \multirow{-5}{*}{\shortstack{Raw\\motion}} & $\underline{.566^{\pm.004}}$ & $\underline{.759^{\pm.003}}$ & $\underline{.849^{\pm.003}}$ & $\mathbf{.038^{\pm.004}}$ & $\underline{2.916^{\pm.008}}$ & $1.518^{\pm.084}$ & $.676^{\pm.001}$ \\
\bottomrule
\end{tabular}}
\end{table}

\Tref{tab:main-results} compares the incremental 263D and global XYZ variants of \method with existing methods using the same 67-dimensional evaluator. The global variant achieves the best FID of $0.038\pm0.004$, whereas the incremental variant attains the strongest point estimates for R-Precision, MM-Dist, and CLIP score. Relative to CMDM, the strongest prior method on most metrics, the global variant reduces FID from $0.078$ to $0.038$ ($51.3\%$), while the incremental variant improves Top-1/2/3 R-Precision by $0.008/0.005/0.004$, reduces MM-Dist by $0.030$, and increases CLIP by $0.001$. Thus, we demonstrate that \method improves both fidelity and text alignment without a learned latent representation. \Tref{tab:main-snap} shows that \method also generalizes to SnapMoGen: it improves Top-1/3 R-Precision over CMDM from $.831/.958$ to $.910/.984$, while retaining a competitive FID of $16.342$ and strong multimodality of $12.538$. Complete metrics and baselines appear in Appendix~\ref{app:snap-results}.

\FloatBarrier

\subsection{Analysis of Representation-Aware Design Choices}

\newsavebox{\compactdesignbox}
\newsavebox{\compactsnapbox}

\begin{lrbox}{\compactsnapbox}
\begin{minipage}[t]{0.45\textwidth}
\centering
\captionof{table}{\textbf{Text-to-motion results on SnapMoGen.} Bold and underline mark the best and second-best generated results.}
\label{tab:main-snap}
\vspace{12pt}
\setlength{\tabcolsep}{1.9pt}
\scriptsize
\resizebox{\linewidth}{!}{%
\begin{tabular}{l|c|cccc}
\toprule
Method & Rep. & Top 1$\uparrow$ & Top 3$\uparrow$ & FID$\downarrow$ & MModality$\uparrow$ \\
\midrule
GT & -- & $.940$ & $.985$ & $.001$ & -- \\
\midrule
T2M-GPT & \multirow{4}{*}{Disc.} & $.618$ & $.812$ & $32.629$ & $9.172$ \\
MoMask & & $.777$ & $.927$ & $17.404$ & $8.183$ \\
MoMask$^{++}$ & & $.802$ & $.938$ & $\underline{15.061}$ & $7.259$ \\
ScaleMoGen & & $.807$ & $.941$ & $16.350$ & $9.399$ \\
\midrule
StableMoFusion & \multirow{4}{*}{Conti.} & $.679$ & $.888$ & $27.801$ & $9.064$ \\
MARDM & & $.648$ & $.856$ & $26.348$ & $9.883$ \\
MotionStreamer & & $.631$ & $.836$ & $30.023$ & $7.543$ \\
CMDM & & $\underline{.831}$ & $\underline{.958}$ & $\mathbf{14.451}$ & $9.521$ \\
\midrule
MDM & & $.503$ & $.727$ & $57.783$ & $\mathbf{13.412}$ \\
\rowcolor[gray]{0.90}\textbf{\method} & \multirow{-2}{*}{Raw} & $\mathbf{.910}$ & $\mathbf{.984}$ & $16.342$ & $\underline{12.538}$ \\
\bottomrule
\end{tabular}}
\end{minipage}
\end{lrbox}

\begin{lrbox}{\compactdesignbox}
\begin{minipage}[t]{0.50\textwidth}
\centering
\captionof{table}{\textbf{Representation-aware design analysis on HumanML3D.} Bold marks the better result in each pair; gray marks the proposed settings.}
\label{tab:design-analysis}
\setlength{\tabcolsep}{1.9pt}
\scriptsize
\resizebox{\linewidth}{!}{%
\begin{tabular}{lll|ccc}
\toprule
Factor & Rep. & Configuration & Top 3$\uparrow$ & FID$\downarrow$ & MM-Dist$\downarrow$ \\
\midrule
\multirow{4}{*}{Attention}
 & \multirow{2}{*}{263D} & \cellcolor[gray]{0.90}Causal, $s=5$ & \cellcolor[gray]{0.90}$\mathbf{.853}$ & \cellcolor[gray]{0.90}$\mathbf{.046}$ & \cellcolor[gray]{0.90}$\mathbf{2.890}$ \\
 & & Bidirectional, $s=5$ & $.846$ & $.067$ & $2.973$ \\
\hhline{~-----}
 & \multirow{2}{*}{XYZ} & Causal, $s=5$ & $.800$ & $1.563$ & $3.243$ \\
 & & \cellcolor[gray]{0.90}Bidirectional, $s=5$ & \cellcolor[gray]{0.90}$\mathbf{.849}$ & \cellcolor[gray]{0.90}$\mathbf{.038}$ & \cellcolor[gray]{0.90}$\mathbf{2.916}$ \\
\midrule
\multirow{4}{*}{Scale}
 & \multirow{2}{*}{263D} & Causal, $s=1$ & $\mathbf{.861}$ & $.111$ & $\mathbf{2.851}$ \\
 & & \cellcolor[gray]{0.90}Causal, $s=5$ & \cellcolor[gray]{0.90}$.853$ & \cellcolor[gray]{0.90}$\mathbf{.046}$ & \cellcolor[gray]{0.90}$2.890$ \\
\hhline{~-----}
 & \multirow{2}{*}{XYZ} & Bidirectional, $s=1$ & $.841$ & $.144$ & $2.999$ \\
 & & \cellcolor[gray]{0.90}Bidirectional, $s=5$ & \cellcolor[gray]{0.90}$\mathbf{.849}$ & \cellcolor[gray]{0.90}$\mathbf{.038}$ & \cellcolor[gray]{0.90}$\mathbf{2.916}$ \\
\midrule
\multirow{4}{*}{Target}
 & \multirow{2}{*}{263D} & \cellcolor[gray]{0.90}Causal, $x$-pred. & \cellcolor[gray]{0.90}$\mathbf{.853}$ & \cellcolor[gray]{0.90}$\mathbf{.046}$ & \cellcolor[gray]{0.90}$\mathbf{2.890}$ \\
 & & Causal, $v$-pred. & $.851$ & $.061$ & $2.900$ \\
\hhline{~-----}
 & \multirow{2}{*}{XYZ} & \cellcolor[gray]{0.90}Bidirectional, $x$-pred. & \cellcolor[gray]{0.90}$\mathbf{.849}$ & \cellcolor[gray]{0.90}$\mathbf{.038}$ & \cellcolor[gray]{0.90}$\mathbf{2.916}$ \\
 & & Bidirectional, $v$-pred. & $.831$ & $.249$ & $3.008$ \\
\bottomrule
\end{tabular}}
\end{minipage}
\end{lrbox}

\begin{table*}[t]
\centering
\usebox{\compactsnapbox}\hfill\usebox{\compactdesignbox}
\end{table*}

\paragraph{Temporal attention depends on the motion representation.}
\label{sec:attention-study}

\Tref{tab:design-analysis} (full version in Appendix~\ref{app:design-analysis}) shows that causal attention reduces 263D FID from $0.067$ to $0.046$, whereas bidirectional attention reduces XYZ FID from $1.563$ to $0.038$. This reversal supports a filtering--smoothing account: causality biases incremental features toward forward-consistent evolution, while bidirectional context helps absolute coordinates maintain whole-trajectory consistency. We interpret this as a finite-model inductive bias rather than a universal advantage of restricted context; formal risk and dynamical analyses appear in Appendix~\ref{app:attention-representation}. Appendix~\ref{app:attention-routing} analyzes the attention patterns for both models and provides attention-routing evidence for this interpretation.

\paragraph{Source scaling substantially affects FID.}
\label{sec:scale-study}

Increasing $s$ from 1 to 5 reduces FID from $0.111$ to $0.046$ for causal 263D generation, with small trade-offs in retrieval performance and MM-Dist, while substantially improving every reported metric for bidirectional XYZ generation. Proposition~\ref{prop:source-geometry} explains why: a larger $s$ improves path isotropy, thereby mitigating anisotropy in motion space by delaying signal emergence. Appendix~\ref{app:motion-correlation} verifies the underlying correlation structure of each representation and quantifies the resulting midpoint SNR and correlation retention. Overall, these results show that the source scale should be selected for the representation rather than inherited from a VAE latent space, making it an important design choice for motion-space generation.

\paragraph{Clean-motion prediction improves distributional fidelity.}
\label{sec:prediction-study}

Clean-motion prediction outperforms direct $v$-prediction on every reported metric in all four matched comparisons. FID improves from $0.061$ to $0.046$ for causal 263D generation and from $0.249$ to $0.038$ for bidirectional XYZ generation. These results show that clean-motion prediction consistently enhances distributional fidelity in high-dimensional motion spaces.

\subsection{Inference-Time Spatial Control}
\label{sec:control-evaluation}

Following the 67-dimensional setting used by ProjFlow~\citep{watanabe2026projflow}, we adopt the OmniControl~\citep{xie2024omnicontrol} protocol with six controlled joints: the pelvis, left foot, right foot, head, left wrist, and right wrist. The protocol evaluates five control densities, corresponding to 1, 2, 5, 49, and 196 keyframes, and averages each metric across these densities. We report FID, Top-3 R-Precision, diversity, foot-skating ratio, trajectory error, location error, and average control error. \Tref{tab:control_noAITS} shows that \method exactly satisfies all evaluated constraints without control-conditioned training. In the all-joints setting, it improves the zero-shot ProjFlow baseline from $0.097$ to $0.061$ FID and from $0.779$ to $0.818$ R-Precision@3. Per-joint results appear in Appendix~\ref{app:control-results}.

\begin{table*}[t]
    \centering
    \caption{\textbf{Results of text-conditioned motion generation with spatial controls on HumanML3D.} The first block trains and evaluates on pelvis controls. The second trains on all joints and reports the average over separately controlled joints; complete per-joint results appear in \Tref{tab:supp_control_full}. Bold and underline denote the best and second-best values, respectively.}
    \vspace{-0.5em}
    \renewcommand{\arraystretch}{0.92}
    \resizebox{1\linewidth}{!}{
    \begin{tabular}{clcccccccc}
    \toprule
       \multirowcell{2}{Controlling\\ Joint} & \multirow{2}{*}{Methods} & \multirow{2}{*}{Zero-shot?}& \multirow{2}{*}{FID$\downarrow$}& R-Precision & \multirow{2}{*}{Diversity$\rightarrow$} & Foot Skating& \multirow{2}{*}{Traj. err.$\downarrow$} & \multirow{2}{*}{Loc. err.$\downarrow$} & \multirow{2}{*}{Avg. err.$\downarrow$}\\
    ~ & && & Top 3 & & Ratio.$\downarrow$ & & & \\
    \midrule
    &GT &- &$0.000$ &$0.795$ &$10.455$ &-&$0.000$ &$0.000$ &$0.000$\\
    \midrule
    \multirow{8}{*}{\parbox{2.2cm}{\centering \textbf{Pelvis}}}
    &MDM~\citep{mdm2022human} &\ding{51} &$1.792$ &$0.673$ &$9.131$ &$0.1019$ &$0.4022$ &$0.3076$ &$0.5959$\\
    &PriorMDM~\citep{shafir2024human} &\ding{55} &$0.393$ &$0.707$ &$9.847$ &$0.0897$&$0.3457$ &$0.2132$ &$0.4417$\\
    &GMD~\citep{karunratanakul2023guided} &\ding{51} &$0.238$ &$0.763$ &$10.011$ &$0.1009$&$0.0931$ &$0.0321$ &$0.1439$\\
    &OmniControl~\citep{xie2024omnicontrol} &\ding{55} &$0.081$ &$0.789$ &$10.323$  &$\underline{0.0547}$ &$0.0387$ &$0.0096$ &$0.0338$\\
    &MotionLCM V2+CtrlNet~\citep{dai2024motionlcm} &\ding{55} &$3.978$ &$0.738$ &$9.249$ &$0.0901$ &$0.1080$ &$0.0581$&$0.1386$ \\
    &MaskControl~\citep{pinyoanuntapong2025maskcontrol} &\ding{55}&$\mathbf{0.066}$ &$\underline{0.799}$ &$\underline{10.474}$ &$\mathbf{0.0543}$ &$\mathbf{0.0000}$ &$\mathbf{0.0000}$& $0.0093$\\
    &ProjFlow~\citep{watanabe2026projflow} &\ding{51} &$0.107$ &$0.784$ &$10.644$ &$0.0629$ &$\mathbf{0.0000}$ &$\mathbf{0.0000}$ &$\mathbf{0.0000}$\\
    &\cellcolor[gray]{0.90}\textbf{\method} &\cellcolor[gray]{0.90}\ding{51} &\cellcolor[gray]{0.90}$\underline{0.068}$ &\cellcolor[gray]{0.90}$\mathbf{0.821}$ &\cellcolor[gray]{0.90}$\mathbf{10.447}$ &\cellcolor[gray]{0.90}$0.0591$ &\cellcolor[gray]{0.90}$\mathbf{0.0000}$ &\cellcolor[gray]{0.90}$\mathbf{0.0000}$ &\cellcolor[gray]{0.90}$\mathbf{0.0000}$\\
    \midrule

    \multirow{5}{*}{\parbox{2.2cm}{\centering \textbf{All Joints\\(Average)}}}
    &OmniControl~\citep{xie2024omnicontrol} &\ding{55} &$0.126$ &$0.792$ &$\underline{10.276}$ & $0.0608$&$0.0617$ &$0.0107$ &$0.0404$\\
    &MotionLCM V2+CtrlNet~\citep{dai2024motionlcm} &\ding{55} &$4.504$ &$0.715$ &$9.230$ &0.1119 &$0.2740$ &$0.1315$ &$0.2464$\\
    &MaskControl~\citep{pinyoanuntapong2025maskcontrol} &\ding{55}&$\underline{0.095}$ &$\underline{0.795}$ &$10.159$ &$\mathbf{0.0545}$ &$\mathbf{0.0000}$ &$\mathbf{0.0000}$& $\underline{0.0065}$\\
    &ProjFlow~\citep{watanabe2026projflow} &\ding{51} &$0.097$ &$0.779$ &$10.651$ &$0.0603$ &$\mathbf{0.0000}$ &$\mathbf{0.0000}$ &$\mathbf{0.0000}$\\
    &\cellcolor[gray]{0.90}\textbf{\method} &\cellcolor[gray]{0.90}\ding{51} &\cellcolor[gray]{0.90}$\mathbf{0.061}$ &\cellcolor[gray]{0.90}$\mathbf{0.818}$ &\cellcolor[gray]{0.90}$\mathbf{10.593}$ &\cellcolor[gray]{0.90}$\underline{0.0578}$ &\cellcolor[gray]{0.90}$\mathbf{0.0000}$ &\cellcolor[gray]{0.90}$\mathbf{0.0000}$ &\cellcolor[gray]{0.90}$\mathbf{0.0000}$\\
    
    \bottomrule
    \end{tabular}}
    \label{tab:control_noAITS}
\end{table*}

\subsection{Qualitative Results}
\label{sec:qualitative-results}

\begin{figure}[t]
    \centering
    \includegraphics[width=\textwidth]{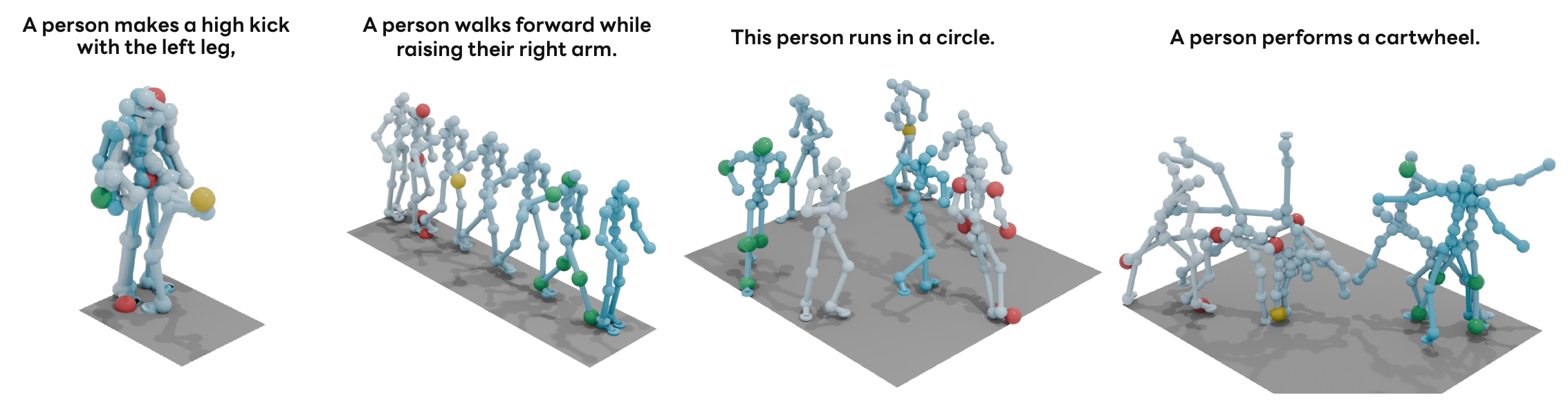}
    \caption{\textbf{Qualitative inference-time spatial control.} Colored markers indicate spatial constraints on selected joints and frames. \method follows the text prompts while satisfying these constraints.}
    \label{fig:qualitative-results}
\end{figure}

In~\Fref{fig:qualitative-results}, we show control over local limb placement, forward locomotion, circular motion, and a cartwheel. In each case, the controlled joints satisfy the constraints exactly, while the unconstrained motion remains coherent and aligned with the prompt. Appendix~\ref{app:qualitative-results} and our project website provides additional text-to-motion comparisons, and the supplementary videos visualize the generated motions.

\section{Limitations}
\label{sec:limitations}

Although \method achieves state-of-the-art text-to-motion performance and exact spatial control, several limitations remain. First, our experiments cover only HumanML3D and SnapMoGen, and whether the observed interplay among representation, source scale, and temporal attention holds across other datasets, skeletons, and motion domains remains to be evaluated. Second, operating directly on full-resolution motion preserves frame-level access but requires processing longer token sequences than temporally compressed latent models, which may limit efficiency for very long motions. Hierarchical or streaming direct-space models could improve scalability while retaining fine-grained control. Finally, our projection sampler enforces linear equality constraints on joint coordinates, but does not explicitly guarantee physical feasibility or support non-linear constraints such as collision avoidance, contact, and joint limits. Extending direct-space generation with physical priors and more expressive constraint solvers is an important direction for future work.

\section{Conclusion}
\label{sec:conclusion}

In this paper, we presented \method, a representation-aware framework for flow matching directly in continuous motion space. \method, which predicts clean motion without a learned encoder or decoder, introduces representation-aware noise scaling for anisotropic direct-space probability paths and uses RA-MMDiT to adapt temporal information flow to the motion representation. Experiments on HumanML3D and SnapMoGen demonstrate state-of-the-art text-to-motion performance and show that causal attention is favored for incremental 263D features, whereas bidirectional attention is essential for absolute XYZ coordinates. The XYZ variant further enables training-free any-joint, any-frame projection with exact constraint satisfaction.

\clearpage

\bibliography{main}
\bibliographystyle{main}

\clearpage
\appendix
\section*{Supplementary Material}

This supplementary material provides additional text-to-motion and spatial-control results, implementation details for reproducing the principal 263D and XYZ models, accompanying sample code, and theoretical and empirical analyses supporting our design choices.

\subsection*{Section Index}

\begingroup
\small
\begin{tabular*}{\linewidth}{@{\extracolsep{\fill}}lr}
\ref*{app:additional-t2m-results}\quad Additional Text-to-Motion Results & \hyperref[app:additional-t2m-results]{p.~\pageref*{app:additional-t2m-results}} \\
\quad\ref*{app:qualitative-results}\quad Qualitative Results & \hyperref[app:qualitative-results]{p.~\pageref*{app:qualitative-results}} \\
\quad\ref*{app:architecture-ablation}\quad Model Capacity and Text--Motion Attention & \hyperref[app:architecture-ablation]{p.~\pageref*{app:architecture-ablation}} \\
\quad\ref*{app:design-analysis}\quad Complete Representation-Aware Design Analysis & \hyperref[app:design-analysis]{p.~\pageref*{app:design-analysis}} \\
\quad\ref*{app:263d-results}\quad Evaluation with the Standard 263D Protocol & \hyperref[app:263d-results]{p.~\pageref*{app:263d-results}} \\
\quad\ref*{app:272d-results}\quad Training and Evaluation with the 272D MotionStreamer Protocol & \hyperref[app:272d-results]{p.~\pageref*{app:272d-results}} \\
\quad\ref*{app:snap-results}\quad Training and Evaluation on SnapMoGen & \hyperref[app:snap-results]{p.~\pageref*{app:snap-results}} \\
\quad\ref*{app:compute-efficiency}\quad Compute Efficiency & \hyperref[app:compute-efficiency]{p.~\pageref*{app:compute-efficiency}} \\
\ref*{app:control-results}\quad Detailed Spatial-Control Results & \hyperref[app:control-results]{p.~\pageref*{app:control-results}} \\
\ref*{app:repro}\quad Implementation and Reproducibility Details & \hyperref[app:repro]{p.~\pageref*{app:repro}} \\
\quad\ref*{app:config-263d}\quad Causal 263D Text-to-Motion Configuration & \hyperref[app:config-263d]{p.~\pageref*{app:config-263d}} \\
\quad\ref*{app:config-xyz}\quad Bidirectional XYZ Text-to-Motion Configuration & \hyperref[app:config-xyz]{p.~\pageref*{app:config-xyz}} \\
\quad\ref*{app:training-sampling}\quad Training and Sampling Algorithms & \hyperref[app:training-sampling]{p.~\pageref*{app:training-sampling}} \\
\quad\ref*{app:rct-control}\quad Bidirectional XYZ Endpoint-Projection Configuration & \hyperref[app:rct-control]{p.~\pageref*{app:rct-control}} \\
\ref*{app:sample-code}\quad Sample Code & \hyperref[app:sample-code]{p.~\pageref*{app:sample-code}} \\
\ref*{app:derivations}\quad Theoretical Details and Proofs & \hyperref[app:derivations]{p.~\pageref*{app:derivations}} \\
\quad\ref*{app:main-proofs}\quad Proofs of Main Propositions & \hyperref[app:main-proofs]{p.~\pageref*{app:main-proofs}} \\
\quad\ref*{app:motion-correlation}\quad Empirical Spatiotemporal Covariance and Source Scale & \hyperref[app:motion-correlation]{p.~\pageref*{app:motion-correlation}} \\
\quad\ref*{app:attention-representation}\quad Why Attention Depends on the Representation & \hyperref[app:attention-representation]{p.~\pageref*{app:attention-representation}} \\
\quad\ref*{app:attention-routing}\quad Empirical Attention Routing & \hyperref[app:attention-routing]{p.~\pageref*{app:attention-routing}} \\
\end{tabular*}
\endgroup

\section{Additional Text-to-Motion Results}
\label{app:additional-t2m-results}

\subsection{Qualitative Results}
\label{app:qualitative-results}

\Fref{fig:supp-t2m-qualitative} compares ACMDM, SALAD, and the global-XYZ and incremental-263D variants of \method on six prompts that probe action count, limb specificity, directed reaching, trajectory reversal, lateral direction, and ground-contact transitions. The red annotations identify visible baseline failures, such as repeating an action, moving the wrong limb or direction, omitting a return trajectory, and floating feet. Both \method variants more consistently preserve the highlighted prompt details while producing coherent motion. Because static trajectory overlays cannot fully convey timing and transition quality, we recommend viewing the accompanying supplementary videos on our project website.

\begin{figure*}[t]
    \centering
    \includegraphics[width=\textwidth]{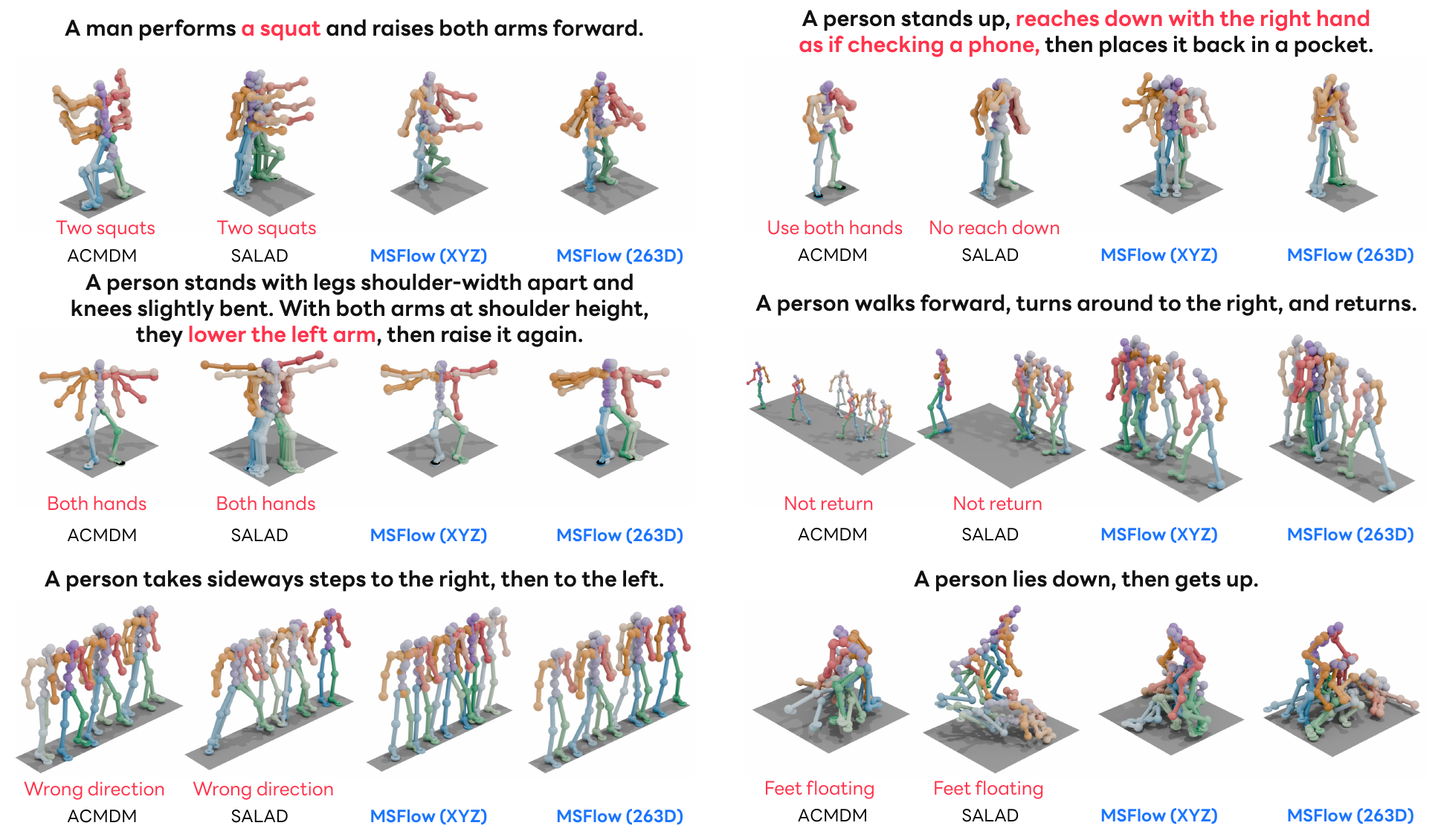}
    \caption{\textbf{Qualitative text-to-motion comparisons on HumanML3D.} We compare ACMDM, SALAD, and the global-XYZ and incremental-263D variants of \method across prompts requiring fine-grained action semantics, directional motion, and contact fidelity. Red annotations identify visible failure modes in the baselines; both \method variants more faithfully follow the highlighted prompt details. \textbf{For the clearest visualization of these dynamic behaviors, please view the accompanying supplementary videos on our project website.}}
    \label{fig:supp-t2m-qualitative}
\end{figure*}

\subsection{Model Capacity and Text--Motion Attention}
\label{app:architecture-ablation}

We ablate model capacity, the flow-time-conditioned Token Refiner, and the attention between text and motion while retaining the selected source scale, prediction target, and representation-specific temporal masks. \Tref{tab:supp-architecture-ablation} reports the complete results for both the causal 263D and bidirectional XYZ generators.

\begin{table*}[t]
\centering
\caption{\textbf{Model-capacity and text--motion attention ablations on HumanML3D.}
All models use clean-motion prediction, source scale $s=5$, classifier-free
guidance weight 3, and the representation-aware temporal mask (causal for
263D and bidirectional for XYZ). MM denotes RA-MMDiT blocks that update
motion and text through multimodal-attention; Cross denotes conventional motion
self-attention followed by text cross-attention. Token Refiner denotes the two-layer,
flow-time-conditioned module applied to text tokens. Each interval summarizes 10
stochastic evaluations. Bold denotes the best point
estimate for each metric within each representation and ablation group. Gray
cells mark the proposed setting.}
\label{tab:supp-architecture-ablation}
\setlength{\tabcolsep}{3.6pt}
\resizebox{\textwidth}{!}{%
\begin{tabular}{lcccc|ccc|ccc}
\toprule
\multicolumn{5}{c|}{Configuration} &
\multicolumn{3}{c|}{R-Precision$\uparrow$} & \multirow{2}{*}{FID$\downarrow$} & \multirow{2}{*}{MM-Dist$\downarrow$} & \multirow{2}{*}{CLIP$\uparrow$} \\
Factor & Rep. & Attention & Blocks/width & Token Refiner & Top 1 & Top 2 & Top 3 & & & \\
\midrule
\multirow{6}{*}{Capacity}
& \multirow{3}{*}{263D}
& MM & $4/256$ & \cmark & $.555^{\pm.003}$ & $.756^{\pm.003}$ & $.848^{\pm.003}$ & $.206^{\pm.008}$ & $2.940^{\pm.011}$ & $.678^{\pm.001}$ \\
& & \cellcolor[gray]{0.90}MM & \cellcolor[gray]{0.90}$8/512$ & \cellcolor[gray]{0.90}\cmark & \cellcolor[gray]{0.90}$\mathbf{.571}^{\pm.005}$ & \cellcolor[gray]{0.90}$\mathbf{.764}^{\pm.003}$ & \cellcolor[gray]{0.90}$\mathbf{.853}^{\pm.003}$ & \cellcolor[gray]{0.90}$\mathbf{.046}^{\pm.004}$ & \cellcolor[gray]{0.90}$\mathbf{2.890}^{\pm.010}$ & \cellcolor[gray]{0.90}$\mathbf{.686}^{\pm.000}$ \\
& & MM & $12/768$ & \cmark & $.561^{\pm.004}$ & $.758^{\pm.003}$ & $.847^{\pm.003}$ & $.076^{\pm.006}$ & $2.925^{\pm.012}$ & $.681^{\pm.001}$ \\
\cmidrule{2-11}
& \multirow{3}{*}{XYZ}
& MM & $4/256$ & \cmark & $.549^{\pm.004}$ & $.750^{\pm.003}$ & $.845^{\pm.003}$ & $.098^{\pm.006}$ & $3.000^{\pm.014}$ & $.658^{\pm.001}$ \\
& & \cellcolor[gray]{0.90}MM & \cellcolor[gray]{0.90}$8/512$ & \cellcolor[gray]{0.90}\cmark & \cellcolor[gray]{0.90}$\mathbf{.566}^{\pm.004}$ & \cellcolor[gray]{0.90}$.759^{\pm.003}$ & \cellcolor[gray]{0.90}$\mathbf{.849}^{\pm.003}$ & \cellcolor[gray]{0.90}$\mathbf{.038}^{\pm.004}$ & \cellcolor[gray]{0.90}$\mathbf{2.916}^{\pm.008}$ & \cellcolor[gray]{0.90}$.676^{\pm.001}$ \\
& & MM & $12/768$ & \cmark & $.564^{\pm.003}$ & $\mathbf{.760}^{\pm.003}$ & $\mathbf{.849}^{\pm.002}$ & $.039^{\pm.003}$ & $2.922^{\pm.006}$ & $\mathbf{.680}^{\pm.001}$ \\
\midrule
\multirow{8}{*}{Attention/Token Refiner}
& \multirow{4}{*}{263D}
& \cellcolor[gray]{0.90}MM & \cellcolor[gray]{0.90}$8/512$ & \cellcolor[gray]{0.90}\cmark & \cellcolor[gray]{0.90}$.571^{\pm.005}$ & \cellcolor[gray]{0.90}$.764^{\pm.003}$ & \cellcolor[gray]{0.90}$.853^{\pm.003}$ & \cellcolor[gray]{0.90}$.046^{\pm.004}$ & \cellcolor[gray]{0.90}$2.890^{\pm.010}$ & \cellcolor[gray]{0.90}$\mathbf{.686}^{\pm.000}$ \\
& & MM & $8/512$ & \xmark & $.568^{\pm.003}$ & $.761^{\pm.003}$ & $.848^{\pm.002}$ & $\mathbf{.043}^{\pm.003}$ & $2.922^{\pm.007}$ & $.679^{\pm.001}$ \\
& & Cross & $8/512$ & \cmark & $.572^{\pm.005}$ & $.762^{\pm.004}$ & $.854^{\pm.003}$ & $.044^{\pm.003}$ & $2.880^{\pm.005}$ & $\mathbf{.686}^{\pm.001}$ \\
& & Cross & $8/512$ & \xmark & $\mathbf{.581}^{\pm.005}$ & $\mathbf{.776}^{\pm.004}$ & $\mathbf{.865}^{\pm.001}$ & $.099^{\pm.006}$ & $\mathbf{2.822}^{\pm.010}$ & $.684^{\pm.001}$ \\
\cmidrule{2-11}
& \multirow{4}{*}{XYZ}
& \cellcolor[gray]{0.90}MM & \cellcolor[gray]{0.90}$8/512$ & \cellcolor[gray]{0.90}\cmark & \cellcolor[gray]{0.90}$.566^{\pm.004}$ & \cellcolor[gray]{0.90}$.759^{\pm.003}$ & \cellcolor[gray]{0.90}$.849^{\pm.003}$ & \cellcolor[gray]{0.90}$\mathbf{.038}^{\pm.004}$ & \cellcolor[gray]{0.90}$2.916^{\pm.008}$ & \cellcolor[gray]{0.90}$.676^{\pm.001}$ \\
& & MM & $8/512$ & \xmark & $.560^{\pm.004}$ & $.754^{\pm.004}$ & $.844^{\pm.002}$ & $.045^{\pm.005}$ & $2.942^{\pm.010}$ & $.675^{\pm.001}$ \\
& & Cross & $8/512$ & \cmark & $.556^{\pm.006}$ & $.756^{\pm.004}$ & $.852^{\pm.002}$ & $.079^{\pm.003}$ & $2.954^{\pm.008}$ & $\mathbf{.677}^{\pm.001}$ \\
& & Cross & $8/512$ & \xmark & $\mathbf{.568}^{\pm.003}$ & $\mathbf{.766}^{\pm.003}$ & $\mathbf{.857}^{\pm.002}$ & $.104^{\pm.004}$ & $\mathbf{2.909}^{\pm.011}$ & $.676^{\pm.001}$ \\
\bottomrule
\end{tabular}}
\end{table*}

The four-block model is consistently under-capacity: relative to the selected eight-block model, its FID increases from $.046$ to $.206$ for 263D and from $.038$ to $.098$ for XYZ, accompanied by weaker retrieval and matching scores. Scaling from eight to 12 blocks does not yield a consistent gain. The 12-block 263D model has worse FID, R-Precision, and MM-Dist, while its XYZ counterpart is statistically similar in FID ($.039\pm.003$ versus $.038\pm.004$) and only improves the CLIP point estimate. We therefore retain the eight-block, width-512 backbone as the best capacity--quality trade-off.

The component ablation reveals a fidelity--retrieval trade-off. Removing the Token Refiner from the multimodal-attention model leaves 263D FID nearly unchanged but weakens its text-alignment metrics; for XYZ it also increases FID from $.038$ to $.045$. Replacing multimodal-attention RA-MMDiT blocks with conventional motion self-attention and text cross-attention is competitive for causal 263D when the Token Refiner is retained, but raises XYZ FID from $.038$ to $.079$. Without the Token Refiner, the cross-attention models obtain the strongest R-Precision and MM-Dist point estimates, yet their FID degrades to $.099$ for 263D and $.104$ for XYZ. Thus, retrieval alone favors the simpler cross-attention variant, whereas multimodal-attention with flow-time-aware text refinement gives the most robust distributional fidelity across both motion representations and is the configuration used by our primary models.

\subsection{Complete Representation-Aware Design Analysis}
\label{app:design-analysis}

\Tref{tab:supp-design-analysis} reports all attention, source-scale, and prediction-target configurations, including Top-1/2/3 R-Precision, FID, MM-Dist, and 95\% confidence intervals. The main paper retains the comparisons needed for its discussion in compact form.

\begin{table}[t]
\centering
\caption{\textbf{Complete analysis of representation-aware design choices.} All configurations use $s=5$ and clean-motion prediction unless stated otherwise. Each interval summarizes 10 stochastic evaluations. Bold denotes the best result for each metric within each group, and gray cells mark the proposed settings.}
\label{tab:supp-design-analysis}
\setlength{\tabcolsep}{5.0pt}
\resizebox{\linewidth}{!}{
\begin{tabular}{lll|ccc|cc}
\toprule
\multirow{2}{*}{Factor} & \multirow{2}{*}{Representation} & \multirow{2}{*}{Configuration} & \multicolumn{3}{c|}{R-Precision$\uparrow$} & \multirow{2}{*}{FID$\downarrow$} & \multirow{2}{*}{MM-Dist$\downarrow$} \\
 & & & Top 1 & Top 2 & Top 3 & & \\
\midrule
\multirow{4}{*}{Attention}
 & \multirow{2}{*}{263D} & \cellcolor[gray]{0.90}Causal & \cellcolor[gray]{0.90}$\mathbf{.571}^{\pm.005}$ & \cellcolor[gray]{0.90}$\mathbf{.764}^{\pm.003}$ & \cellcolor[gray]{0.90}$\mathbf{.853}^{\pm.003}$ & \cellcolor[gray]{0.90}$\mathbf{.046}^{\pm.004}$ & \cellcolor[gray]{0.90}$\mathbf{2.890}^{\pm.010}$ \\
 & & Bidirectional & $.558^{\pm.004}$ & $.756^{\pm.003}$ & $.846^{\pm.002}$ & $.067^{\pm.002}$ & $2.973^{\pm.005}$ \\
\cmidrule{2-8}
 & \multirow{2}{*}{XYZ} & Causal & $.506^{\pm.004}$ & $.701^{\pm.002}$ & $.800^{\pm.002}$ & $1.563^{\pm.025}$ & $3.243^{\pm.008}$ \\
 & & \cellcolor[gray]{0.90}Bidirectional & \cellcolor[gray]{0.90}$\mathbf{.566}^{\pm.004}$ & \cellcolor[gray]{0.90}$\mathbf{.759}^{\pm.003}$ & \cellcolor[gray]{0.90}$\mathbf{.849}^{\pm.003}$ & \cellcolor[gray]{0.90}$\mathbf{.038}^{\pm.004}$ & \cellcolor[gray]{0.90}$\mathbf{2.916}^{\pm.008}$ \\
\midrule
\multirow{4}{*}{Source scale}
 & \multirow{2}{*}{263D} & $s=1$, causal & $\mathbf{.580}^{\pm.004}$ & $\mathbf{.775}^{\pm.003}$ & $\mathbf{.861}^{\pm.002}$ & $.111^{\pm.005}$ & $\mathbf{2.851}^{\pm.007}$ \\
 & & \cellcolor[gray]{0.90}$s=5$, causal & \cellcolor[gray]{0.90}$.571^{\pm.005}$ & \cellcolor[gray]{0.90}$.764^{\pm.003}$ & \cellcolor[gray]{0.90}$.853^{\pm.003}$ & \cellcolor[gray]{0.90}$\mathbf{.046}^{\pm.004}$ & \cellcolor[gray]{0.90}$2.890^{\pm.010}$ \\
\cmidrule{2-8}
 & \multirow{2}{*}{XYZ} & $s=1$, bidirectional & $.548^{\pm.002}$ & $.748^{\pm.003}$ & $.841^{\pm.003}$ & $.144^{\pm.009}$ & $2.999^{\pm.008}$ \\
 & & \cellcolor[gray]{0.90}$s=5$, bidirectional & \cellcolor[gray]{0.90}$\mathbf{.566}^{\pm.004}$ & \cellcolor[gray]{0.90}$\mathbf{.759}^{\pm.003}$ & \cellcolor[gray]{0.90}$\mathbf{.849}^{\pm.003}$ & \cellcolor[gray]{0.90}$\mathbf{.038}^{\pm.004}$ & \cellcolor[gray]{0.90}$\mathbf{2.916}^{\pm.008}$ \\
\midrule
\multirow{8}{*}{Prediction target}
 & \multirow{4}{*}{263D} & \cellcolor[gray]{0.90}$x$-prediction, causal & \cellcolor[gray]{0.90}$\mathbf{.571}^{\pm.005}$ & \cellcolor[gray]{0.90}$\mathbf{.764}^{\pm.003}$ & \cellcolor[gray]{0.90}$\mathbf{.853}^{\pm.003}$ & \cellcolor[gray]{0.90}$\mathbf{.046}^{\pm.004}$ & \cellcolor[gray]{0.90}$\mathbf{2.890}^{\pm.010}$ \\
 & & $v$-prediction, causal & $.567^{\pm.003}$ & $.763^{\pm.003}$ & $.851^{\pm.003}$ & $.061^{\pm.004}$ & $2.900^{\pm.011}$ \\
 & & $x$-prediction, bidirectional & $.558^{\pm.004}$ & $.756^{\pm.003}$ & $.846^{\pm.002}$ & $.067^{\pm.002}$ & $2.973^{\pm.005}$ \\
 & & $v$-prediction, bidirectional & $.551^{\pm.004}$ & $.752^{\pm.003}$ & $.842^{\pm.002}$ & $.101^{\pm.004}$ & $3.010^{\pm.009}$ \\
\cmidrule{2-8}
 & \multirow{4}{*}{XYZ} & $x$-prediction, causal & $.506^{\pm.004}$ & $.701^{\pm.002}$ & $.800^{\pm.002}$ & $1.563^{\pm.025}$ & $3.243^{\pm.008}$ \\
 & & $v$-prediction, causal & $.450^{\pm.002}$ & $.637^{\pm.004}$ & $.744^{\pm.004}$ & $4.482^{\pm.030}$ & $3.672^{\pm.012}$ \\
 & & \cellcolor[gray]{0.90}$x$-prediction, bidirectional & \cellcolor[gray]{0.90}$\mathbf{.566}^{\pm.004}$ & \cellcolor[gray]{0.90}$\mathbf{.759}^{\pm.003}$ & \cellcolor[gray]{0.90}$\mathbf{.849}^{\pm.003}$ & \cellcolor[gray]{0.90}$\mathbf{.038}^{\pm.004}$ & \cellcolor[gray]{0.90}$\mathbf{2.916}^{\pm.008}$ \\
 & & $v$-prediction, bidirectional & $.542^{\pm.004}$ & $.737^{\pm.002}$ & $.831^{\pm.003}$ & $.249^{\pm.014}$ & $3.008^{\pm.010}$ \\
\bottomrule
\end{tabular}}
\end{table}

\subsection{Evaluation with the Standard 263D Protocol}
\label{app:263d-results}

We complement the 67D evaluation in the main paper with the standard 263D HumanML3D evaluator used by MDM~\citep{mdm2022human}. \Tref{tab:supp-t2m-263d} shows that the incremental variant achieves the best Top-1 and Top-3 R-Precision scores of $0.591$ and $0.863$, respectively, and the best MM-Dist of $2.565$, while tying CMDM for the best Top-2 R-Precision of $0.778$. The global variant obtains the second-best FID of $0.050$ and the second-best MM-Dist of $2.607$. Together, these results demonstrate strong retrieval and distributional performance under both evaluator protocols.

\begin{table}[t]
\centering
\caption{\textbf{Text-to-motion results on HumanML3D using the standard 263D evaluator.} \method (263D) denotes the causal incremental-263D variant, while \method (XYZ) denotes the bidirectional global-XYZ variant; both use clean-motion prediction and $s=5$. All intervals are 95\% confidence intervals. \textbf{Bold} indicates the best result, and \underline{underline} denotes the second-best result.}
\label{tab:supp-t2m-263d}
\setlength{\tabcolsep}{2.8pt}
\resizebox{\linewidth}{!}{
\begin{tabular}{l|c|ccc|ccc}
\toprule
\multirow{2}{*}{Method} & \multirow{2}{*}{Representation} & \multicolumn{3}{c|}{R-Precision$\uparrow$} & \multirow{2}{*}{FID$\downarrow$} & \multirow{2}{*}{MM-Dist$\downarrow$} & \multirow{2}{*}{MModality$\uparrow$} \\
\cline{3-5}
~ & ~ & Top 1 & Top 2 & Top 3 & & & \\
\midrule
GT & -- & $.511^{\pm.003}$ & $.703^{\pm.003}$ & $.797^{\pm.002}$ & $.002^{\pm.000}$ & $2.974^{\pm.008}$ & -- \\
\midrule
T2M-GPT~\citep{zhang2023t2m} & \multirow{4}{*}{\shortstack{Discrete\\latent}} & $.492^{\pm.003}$ & $.679^{\pm.002}$ & $.775^{\pm.002}$ & $.141^{\pm.005}$ & $3.121^{\pm.009}$ & $1.831^{\pm.048}$ \\
MMM~\citep{pinyoanuntapong2024mmm} & & $.515^{\pm.002}$ & $.708^{\pm.002}$ & $.804^{\pm.002}$ & $.089^{\pm.005}$ & $2.926^{\pm.007}$ & $1.226^{\pm.035}$ \\
MoMask~\citep{guo2024momask} & & $.521^{\pm.002}$ & $.713^{\pm.002}$ & $.807^{\pm.002}$ & $\mathbf{.045^{\pm.002}}$ & $2.958^{\pm.008}$ & $1.241^{\pm.040}$ \\
IRG-MotionLLM~\citep{li2025irg} & & $.564^{\pm.002}$ & $.754^{\pm.002}$ & $.841^{\pm.002}$ & $.208^{\pm.003}$ & $2.628^{\pm.005}$ & -- \\
\midrule
MLD V2~\citep{chen2023executing} & \multirow{9}{*}{\shortstack{Continuous\\latent}} & $.542^{\pm.002}$ & $.735^{\pm.002}$ & $.827^{\pm.002}$ & $.078^{\pm.004}$ & $2.808^{\pm.006}$ & $1.676^{\pm.060}$ \\
MotionLCM V2~\citep{dai2024motionlcm} & & $.548^{\pm.002}$ & $.743^{\pm.002}$ & $.835^{\pm.002}$ & $.092^{\pm.003}$ & $2.760^{\pm.008}$ & $1.800^{\pm.047}$ \\
StableMoFusion~\citep{huang2024stablemofusion} & & $.553^{\pm.003}$ & $.748^{\pm.002}$ & $.841^{\pm.002}$ & $.098^{\pm.003}$ & $2.715^{\pm.006}$ & $1.774^{\pm.051}$ \\
EnergyMoGen~\citep{zhang2025energymogen} & & $.523^{\pm.003}$ & $.715^{\pm.002}$ & $.815^{\pm.002}$ & $.188^{\pm.006}$ & $2.915^{\pm.007}$ & $\underline{2.205^{\pm.041}}$ \\
SALAD~\citep{hong2025salad} & & $.581^{\pm.003}$ & $.769^{\pm.003}$ & $.857^{\pm.002}$ & $.076^{\pm.002}$ & $2.649^{\pm.009}$ & $1.751^{\pm.062}$ \\
MARDM~\citep{meng2024rethinking} & & $.517^{\pm.003}$ & $.708^{\pm.003}$ & $.805^{\pm.003}$ & $.116^{\pm.006}$ & $2.968^{\pm.010}$ & $1.923^{\pm.105}$ \\
MotionStreamer~\citep{xiao2025motionstreamer} & & $.496^{\pm.003}$ & $.695^{\pm.002}$ & $.793^{\pm.002}$ & $.201^{\pm.005}$ & $3.041^{\pm.009}$ & $1.463^{\pm.075}$ \\
FloodDiffusion~\citep{cai2026flooddiffusion} & & $.523^{\pm.002}$ & $.717^{\pm.002}$ & $.810^{\pm.003}$ & $.057^{\pm.002}$ & $2.887^{\pm.007}$ & -- \\
CMDM~\citep{yu2026causal} & & $\underline{.588^{\pm.004}}$ & $\mathbf{.778^{\pm.002}}$ & $\underline{.860^{\pm.003}}$ & $.068^{\pm.003}$ & $2.620^{\pm.010}$ & $1.785^{\pm.074}$ \\
\midrule
MDM~\citep{mdm2022human} & & $.455^{\pm.006}$ & $.645^{\pm.007}$ & $.749^{\pm.002}$ & $.489^{\pm.025}$ & $3.330^{\pm.025}$ & $\mathbf{2.290^{\pm.070}}$ \\
MotionDiffuse~\citep{zhang2022motiondiffuse} & & $.491^{\pm.001}$ & $.681^{\pm.001}$ & $.782^{\pm.001}$ & $.630^{\pm.001}$ & $3.113^{\pm.001}$ & $1.553^{\pm.042}$ \\
ReMoDiffuse~\citep{zhang2023remodiffuse} & & $.510^{\pm.005}$ & $.698^{\pm.006}$ & $.795^{\pm.004}$ & $.103^{\pm.004}$ & $2.974^{\pm.016}$ & $1.795^{\pm.043}$ \\
\rowcolor[gray]{0.90}\textbf{\method (263D)} & & $\mathbf{.591^{\pm.004}}$ & $\mathbf{.778^{\pm.004}}$ & $\mathbf{.863^{\pm.003}}$ & $.057^{\pm.003}$ & $\mathbf{2.565^{\pm.007}}$ & $.890^{\pm.048}$ \\
\rowcolor[gray]{0.90}\textbf{\method (XYZ)} & \multirow{-5}{*}{\shortstack{Raw\\motion}} & $.581^{\pm.003}$ & $\underline{.773^{\pm.003}}$ & $.858^{\pm.002}$ & $\underline{.050^{\pm.004}}$ & $\underline{2.607^{\pm.009}}$ & $1.287^{\pm.099}$ \\
\bottomrule
\end{tabular}}
\end{table}

\subsection{Training and Evaluation with the 272D MotionStreamer Protocol}
\label{app:272d-results}

To compare under the protocol introduced by MotionStreamer~\citep{xiao2025motionstreamer}, we train a new causal \method model from scratch on the 272D HumanML3D representation used by MotionStreamer, rather than converting or fine-tuning our 263D model. We use the proposed architecture with source scale $s=5$ and causal attention because the 272D representation, like the 263D representation, encodes incremental motion. We evaluate the model using the same official 272D evaluator as MotionStreamer and UMO~\citep{cong2026umo}.

\begin{table}[t]
\centering
\caption{\textbf{Text-to-motion results on HumanML3D using the MotionStreamer 272D evaluator.} We include all baselines reported by MotionStreamer and UMO. UMO converts the 201D outputs of HY-Motion and its own models to 272D for evaluation. Our causal 272D model is trained from scratch with clean-motion prediction and $s=5$. Its intervals are 95\% confidence intervals. \textbf{Bold} indicates the best generated result.}
\label{tab:supp-t2m-272d}
\setlength{\tabcolsep}{4.5pt}
\resizebox{\linewidth}{!}{
\begin{tabular}{l|c|c|ccc|c}
\toprule
\multirow{2}{*}{Method} & \multirow{2}{*}{Representation} & \multirow{2}{*}{FID$\downarrow$} & \multicolumn{3}{c|}{R-Precision$\uparrow$} & \multirow{2}{*}{MM-Dist$\downarrow$} \\
\cline{4-6}
~ & ~ & ~ & Top 1 & Top 2 & Top 3 & \\
\midrule
Real motion & -- & $.002$ & $.702$ & $.864$ & $.914$ & $15.151$ \\
\midrule
T2M-GPT~\citep{zhang2023t2m} & \multirow{4}{*}{\shortstack{Discrete\\latent}} & $12.475$ & $.606$ & $.774$ & $.838$ & $16.812$ \\
MotionGPT~\citep{jiang2023motiongpt} & & $14.375$ & $.456$ & $.598$ & $.628$ & $17.892$ \\
MoMask~\citep{guo2024momask} & & $12.232$ & $.621$ & $.784$ & $.846$ & $16.138$ \\
AttT2M~\citep{zhong2023attt2m} & & $15.428$ & $.592$ & $.765$ & $.834$ & $15.726$ \\
\midrule
MLD~\citep{chen2023executing} & \multirow{2}{*}{\shortstack{Continuous\\latent}} & $18.236$ & $.546$ & $.730$ & $.792$ & $16.638$ \\
MotionStreamer~\citep{xiao2025motionstreamer} & & $11.790$ & $.631$ & $.802$ & $.859$ & $16.081$ \\
\midrule
MDM~\citep{mdm2022human} & & $23.454$ & $.523$ & $.692$ & $.764$ & $17.423$ \\
HY-Motion~\citep{wen2025hymotion} & & $61.035$ & $.667$ & $.818$ & $.876$ & $17.530$ \\
UMO-Expert~\citep{cong2026umo} & & $17.040$ & $.763$ & $.889$ & $.931$ & $15.490$ \\
UMO-Unified~\citep{cong2026umo} & & $9.460$ & $.774$ & $.892$ & $.933$ & $15.220$ \\
\rowcolor[gray]{0.90}\textbf{\method} & \multirow{-5}{*}{\shortstack{Raw\\motion}} & $\mathbf{6.540^{\pm.120}}$ & $\mathbf{.792^{\pm.003}}$ & $\mathbf{.919^{\pm.002}}$ & $\mathbf{.955^{\pm.002}}$ & $\mathbf{13.864^{\pm.014}}$ \\
\bottomrule
\end{tabular}}
\end{table}

As shown in~\Tref{tab:supp-t2m-272d}, the new 272D \method model improves FID from $11.790$ for MotionStreamer and $9.460$ for UMO-Unified to $6.540$. It also increases Top-1/2/3 R-Precision to $0.792/0.919/0.955$ and reduces MM-Dist to $13.864$, outperforming both UMO variants on every reported metric. These comparisons use the same evaluator, although their generation spaces differ: MotionStreamer and our model are trained on 272D motion, whereas UMO generates a 201D HY-Motion representation and converts its output to 272D for evaluation.

\paragraph{Difference from the standard 263D representation.}
For $K=22$ SMPL joints, the standard HumanML3D vector contains planar root velocity ($2$D), scalar yaw velocity ($1$D), root height ($1$D), root-relative positions for the $K-1$ non-root joints ($3(K-1)$D), local velocities for all joints ($3K$D), IK-derived 6D rotations for the non-root joints ($6(K-1)$D), and four foot-contact channels, totaling $263$ dimensions. The 272D representation instead contains planar root velocity ($2$D), a 6D root angular increment, and positions, velocities, and 6D rotations for all $K$ joints ($3K+3K+6K$D), totaling $272$ dimensions. Thus, it removes explicit root height and foot-contact channels, includes the root in the position and rotation blocks, and replaces scalar yaw velocity with a 6D rotation increment. More importantly, its joint rotations come directly from the source SMPL motion rather than being recovered from positions by inverse kinematics. This retains twist information and permits direct SMPL animation without the slow, error-prone SMPLify post-processing required when only positions from the standard representation are used.

\subsection{Training and Evaluation on SnapMoGen}
\label{app:snap-results}

We additionally train an \method model from scratch on SnapMoGen~\citep{guo2025snapmogen} using its native 296D motion representation with source scale $s=10$. We evaluate the results using the same protocol used by previous works~\citep{guo2025snapmogen}.

\paragraph{Native motion representation and BVH recovery.}
Each SnapMoGen frame contains root-yaw velocity ($1$D), planar root velocity ($2$D), root height ($1$D), heading-normalized global 6D rotations ($24\times6$D), joint positions ($24\times3$D), joint velocities ($24\times3$D), and four foot-contact indicators, totaling 296 dimensions. Unlike HumanML3D's parent-relative rotations, SnapMoGen rotations are global after removing root heading. BVH recovery requires only the first 148 channels: integrating root motion recovers the trajectory, while converting 6D rotations to quaternions, restoring heading, and applying the template hierarchy yields local BVH rotations. This also differs from evaluation on HumanML3D: its 67D protocol embeds four root channels and root-relative XYZ positions for 21 non-root joints ($4+21\times3=67$), whereas the official SnapMoGen evaluator embeds the first 148 root-and-rotation channels. Our model nevertheless trains on and generates all 296 channels, with the remaining channels providing redundant learning signals.

\paragraph{Attention graph.}
HumanML3D's incremental features naturally support forward accumulation and therefore favor causal attention. In contrast, SnapMoGen uses a hybrid representation that combines root motion and global joint rotations with redundant position, velocity, and contact features.  Bidirectional attention uses context from both directions to keep these features consistent across the full motion, and it outperforms causal attention on SnapMoGen. SnapMoGen also uses longer sequences (312 versus 192 frames), amplifying the context asymmetry of a causal mask. We therefore use bidirectional attention, while viewing this design as an interaction among representation, sequence length, and source scale rather than a universal advantage.

\begin{table}[t]
\centering
\caption{\textbf{Text-to-motion results on SnapMoGen.}  Our result uses a model trained from scratch on the native 296D SnapMoGen representation with clean-motion prediction and $s=10$. All intervals are 95\% confidence intervals over 10 runs. \textbf{Bold} and \underline{underline} indicate the best and second-best generated results.}
\label{tab:supp-t2m-snap}
\setlength{\tabcolsep}{3.2pt}
\resizebox{\linewidth}{!}{
\begin{tabular}{l|c|ccc|cc}
\toprule
\multirow{2}{*}{Method} & \multirow{2}{*}{Representation} & \multicolumn{3}{c|}{R-Precision$\uparrow$} & \multirow{2}{*}{FID$\downarrow$} & \multirow{2}{*}{MModality$\uparrow$} \\
\cline{3-5}
~ & ~ & Top 1 & Top 2 & Top 3 & & \\
\midrule
GT & -- & $.940^{\pm.001}$ & $.976^{\pm.001}$ & $.985^{\pm.001}$ & $.001^{\pm.000}$ & -- \\
\midrule
T2M-GPT~\citep{zhang2023t2m} & \multirow{4}{*}{\shortstack{Discrete\\latent}} & $.618^{\pm.002}$ & $.773^{\pm.002}$ & $.812^{\pm.002}$ & $32.629^{\pm.087}$ & $9.172^{\pm.181}$ \\
MoMask~\citep{guo2024momask} & & $.777^{\pm.002}$ & $.888^{\pm.002}$ & $.927^{\pm.002}$ & $17.404^{\pm.051}$ & $8.183^{\pm.184}$ \\
MoMask$^{++}$~\citep{guo2025snapmogen} & & $.802^{\pm.001}$ & $.905^{\pm.002}$ & $.938^{\pm.001}$ & $\underline{15.061^{\pm.065}}$ & $7.259^{\pm.180}$ \\
ScaleMoGen~\citep{hwang2026scalemogen} & & $.807^{\pm.004}$ & $.908^{\pm.003}$ & $.941^{\pm.003}$ & $16.350^{\pm.086}$ & $9.399^{\pm.669}$ \\
\midrule
StableMoFusion~\citep{huang2024stablemofusion} & \multirow{4}{*}{\shortstack{Continuous\\latent}} & $.679^{\pm.002}$ & $.823^{\pm.002}$ & $.888^{\pm.002}$ & $27.801^{\pm.063}$ & $9.064^{\pm.138}$ \\
MARDM~\citep{meng2024rethinking} & & $.648^{\pm.002}$ & $.801^{\pm.002}$ & $.856^{\pm.002}$ & $26.348^{\pm.208}$ & $9.883^{\pm.147}$ \\
MotionStreamer~\citep{xiao2025motionstreamer} & & $.631^{\pm.002}$ & $.791^{\pm.002}$ & $.836^{\pm.002}$ & $30.023^{\pm.131}$ & $7.543^{\pm.195}$ \\
CMDM~\citep{yu2026causal} & & $\underline{.831^{\pm.004}}$ & $\underline{.926^{\pm.003}}$ & $\underline{.958^{\pm.002}}$ & $\mathbf{14.451^{\pm.089}}$ & $9.521^{\pm.196}$ \\
\midrule
MDM~\citep{mdm2022human} & & $.503^{\pm.002}$ & $.653^{\pm.002}$ & $.727^{\pm.002}$ & $57.783^{\pm.092}$ & $\mathbf{13.412^{\pm.231}}$ \\
\rowcolor[gray]{0.90}\textbf{\method} & \multirow{-2}{*}{\shortstack{Raw\\motion}} & $\mathbf{.910^{\pm.002}}$ & $\mathbf{.969^{\pm.002}}$ & $\mathbf{.984^{\pm.001}}$ & $16.342^{\pm.118}$ & $\underline{12.538^{\pm.591}}$ \\
\bottomrule
\end{tabular}}
\end{table}

\Tref{tab:supp-t2m-snap} shows that \method obtains the strongest text--motion alignment, improving Top-1/2/3 R-Precision from $0.831/0.926/0.958$ for CMDM to $0.910/0.969/0.984$ and exceeding ScaleMoGen's $0.807/0.908/0.941$. Its FID of $16.342$ is higher than CMDM ($14.451$) and MoMask$^{++}$ ($15.061$), but slightly lower than ScaleMoGen ($16.350$) and all other remaining baselines, including MoMask ($17.404$). \method also obtains the second-best multimodality score of $12.538$, behind MDM ($13.412$). The result therefore improves semantic retrieval and distributional fidelity over the previous SnapMoGen checkpoint while retaining strong conditional diversity.

\subsection{Compute Efficiency}
\label{app:compute-efficiency}

We profile \method and the selected prior methods under a common 196-frame generation protocol. For each method, we use its official classifier-free guidance and sampling schedule and measure the complete sampling trajectory, including decoding into motion space. We retain each release's native motion representation: MARDM generates 67D features, MotionStreamer generates 272D features, and the remaining methods generate 263D features. All runs use batch size one, FP32 with TF32 enabled, and the prompt ``a person walks forward.'' on a single NVIDIA A100-SXM4-80GB GPU. We report the median of five synchronized runs after two warm-up runs. FLOPs are measured with the PyTorch flop counter, with a multiply--add counted as two operations.

\begin{table}[t]
\centering
\caption{\textbf{Quality and compute efficiency for generating a 196-frame motion.} Parameters include the complete learned generation stack and motion encoder, where applicable, but exclude frozen text encoders. FLOPs count a multiply--add as two operations. Time is the median of five runs on one NVIDIA A100-SXM4-80GB GPU. FID and R-Top3 are point estimates from the standard 263D evaluation in~\Tref{tab:supp-t2m-263d}. \textbf{Bold} and \underline{underline} denote the best and second-best available results among the methods shown.}
\label{tab:supp-compute-efficiency}
\resizebox{0.9\columnwidth}{!}{%
\begin{tabular}{lrrrrr}
\toprule
Method & Params (M) $\downarrow$ & FLOPs (TF) $\downarrow$ & Time (s) $\downarrow$ & FID $\downarrow$ & R-Top3 $\uparrow$ \\
\midrule
MARDM~\citep{meng2024rethinking} & 309.65 & 75.477 & 9.579 & 0.116 & 0.805 \\
SALAD~\citep{hong2025salad} & \textbf{10.06} & \textbf{0.467} & \textbf{0.477} & 0.076 & 0.857 \\
MotionStreamer~\citep{xiao2025motionstreamer} & 318.51 & 1.959 & 9.023 & 0.201 & 0.793 \\
CMDM~\citep{yu2026causal} & 115.01 & \underline{0.572} & 1.809 & \underline{0.068} & \underline{0.860} \\
FloodDiffusion~\citep{cai2026flooddiffusion} & 131.22 & 7.060 & 3.537 & \textbf{0.057} & 0.810 \\
\rowcolor[gray]{0.90}
\textbf{\method} & \underline{68.77} & 1.502 & \underline{1.742} & \textbf{0.057} & \textbf{0.863} \\
\bottomrule
\end{tabular}
}
\end{table}

For model size, we count all learned modules required for motion generation, including the full motion encoder when a method uses one, but omit frozen text encoders because they perform cacheable, one-time prompt preprocessing outside the iterative generator. Their execution is likewise excluded from FLOPs and timing. MARDM uses its official trained checkpoint and the actual cached CLIP embedding because its adaptive Dopri5 solver has conditioning-dependent function evaluations. FloodDiffusion internally requires 50 latent tokens, which decode to 197 frames; we follow its length construction and crop the final decoded frame so that every reported output contains exactly 196 frames.

As shown in~\Tref{tab:supp-compute-efficiency}, SALAD is the most efficient method on all three compute measures. \method is second-smallest at 68.77M parameters and second-fastest at 1.742 seconds, while its 1.502 TFLOPs is the third-lowest after SALAD and CMDM. Using the standard 263D evaluation from~\Tref{tab:supp-t2m-263d}, it obtains the highest R-Top3 of 0.863 and ties FloodDiffusion for the best FID of 0.057. Relative to CMDM, it uses 40.2\% fewer parameters and is 3.7\% faster, while improving R-Top3 by 0.003 and FID by 0.011, although it requires more FLOPs. It is also faster and uses fewer FLOPs than MotionStreamer, FloodDiffusion, and MARDM. Thus, direct full-sequence motion generation remains competitive in model size and latency while attaining the strongest reported quality.

\section{Detailed Spatial-Control Results}
\label{app:control-results}
\begin{table}[t]
    \centering
    \caption{\textbf{Complete results for text-conditioned spatial control and upper-body editing on HumanML3D.} The first block trains and evaluates on pelvis controls. The per-joint blocks use methods trained on all joints. The final block follows the upper-body editing format of ProjFlow~\citep{watanabe2026projflow}. Bold and underline denote the best and second-best values, respectively; diversity is ranked by absolute deviation from GT, ties share the same formatting, and gray cells mark \method.}
    \renewcommand{\arraystretch}{1.0}
    \resizebox{1\linewidth}{!}{
    \begin{tabular}{clcccccccc}
    \toprule
       \multirowcell{2}{Controlling\\ Joint} & \multirow{2}{*}{Methods} & \multirow{2}{*}{Zero-shot?}& \multirow{2}{*}{FID$\downarrow$}& R-Precision & \multirow{2}{*}{Diversity$\rightarrow$} & Foot Skating& \multirow{2}{*}{Traj. err.$\downarrow$} & \multirow{2}{*}{Loc. err.$\downarrow$} & \multirow{2}{*}{Avg. err.$\downarrow$}\\
    ~ & &&& Top 3 & & Ratio.$\downarrow$ & & & \\
    \midrule
    &GT &$-$ &$0.000$ &$0.795$ &$10.455$ &-&$0.000$ &$0.000$ &$0.000$\\
    \midrule
    \multirow{6}{*}{\parbox{2.2cm}{\centering \textbf{Train\\On\\Pelvis}}}
    &MDM~\citep{mdm2022human} &\ding{51} &$1.792$ &$0.673$ &$9.131$ &$0.1019$ &$0.4022$ &$0.3076$ &$0.5959$\\
    &PriorMDM~\citep{shafir2024human} &\ding{55} &$0.393$ &$0.707$ &$9.847$ &$0.0897$&$0.3457$ &$0.2132$ &$0.4417$\\
    &GMD~\citep{karunratanakul2023guided}  &\ding{51} &$0.238$ &$0.763$ &$10.011$ &$0.1009$&$0.0931$ &$0.0321$ &$0.1439$\\
    &OmniControl~\citep{xie2024omnicontrol} &\ding{55} &$\underline{0.081}$ &$\underline{0.789}$ &$\underline{10.323}$  &$\mathbf{0.0547}$ &$\underline{0.0387}$ &$\underline{0.0096}$ &$\underline{0.0338}$\\
    &ProjFlow~\citep{watanabe2026projflow} &\ding{51} &$0.107$ &$0.784$ &$10.645$ &$0.0630$ &$\mathbf{0.0000}$ &$\mathbf{0.0000}$ &$\mathbf{0.0000}$\\
    &\cellcolor[gray]{0.90}\textbf{\method} &\cellcolor[gray]{0.90}\ding{51} &\cellcolor[gray]{0.90}$\mathbf{0.068}$ &\cellcolor[gray]{0.90}$\mathbf{0.821}$ &\cellcolor[gray]{0.90}$\mathbf{10.447}$ &\cellcolor[gray]{0.90}$\underline{0.0591}$ &\cellcolor[gray]{0.90}$\mathbf{0.0000}$ &\cellcolor[gray]{0.90}$\mathbf{0.0000}$ &\cellcolor[gray]{0.90}$\mathbf{0.0000}$\\
    \midrule
    \midrule

    \multirow{5}{*}{\textbf{Pelvis}}
    &OmniControl~\citep{xie2024omnicontrol} &\ding{55} &$0.135$ &$0.790$ &$\underline{10.314}$ &$\underline{0.0571}$ &$0.0404$ &$\underline{0.0085}$ &$0.0367$\\
    &MotionLCM V2+CtrlNet~\citep{dai2024motionlcm} &\ding{55}&$4.726$&$0.713$&$9.209$ &$0.1162$&$0.1617$ &$0.0841$ &$0.1838$\\
    &MaskControl~\citep{pinyoanuntapong2025maskcontrol} &\ding{55} &$\underline{0.087}$ &$\underline{0.795}$ &$10.168$ &$\mathbf{0.0544}$ &$\underline{0.0003}$ &$\mathbf{0.0000}$ &$\underline{0.0114}$\\
    &ProjFlow~\citep{watanabe2026projflow} &\ding{51} &$0.107$ &$0.784$ &$10.645$ &$0.0630$ &$\mathbf{0.0000}$ &$\mathbf{0.0000}$ &$\mathbf{0.0000}$\\
    &\cellcolor[gray]{0.90}\textbf{\method} &\cellcolor[gray]{0.90}\ding{51} &\cellcolor[gray]{0.90}$\mathbf{0.068}$ &\cellcolor[gray]{0.90}$\mathbf{0.821}$ &\cellcolor[gray]{0.90}$\mathbf{10.447}$ &\cellcolor[gray]{0.90}$0.0591$ &\cellcolor[gray]{0.90}$\mathbf{0.0000}$ &\cellcolor[gray]{0.90}$\mathbf{0.0000}$ &\cellcolor[gray]{0.90}$\mathbf{0.0000}$\\
    \midrule

    \multirow{5}{*}{\textbf{Left foot}}
    &OmniControl~\citep{xie2024omnicontrol} &\ding{55} &$0.093$ &$\underline{0.794}$ &$\mathbf{10.338}$ &$0.0692$&$\underline{0.0594}$ &$\underline{0.0094}$ &$0.0314$\\
    &MotionLCM V2+CtrlNet~\citep{dai2024motionlcm} &\ding{55}&$4.810$ &$0.706$ &$9.158$ &$0.1047$&$0.2607$ &$0.1229$ &$0.2304$\\
    &MaskControl~\citep{pinyoanuntapong2025maskcontrol} &\ding{55} &$\underline{0.074}$ &$0.793$ &$10.241$ &$\mathbf{0.0561}$ &$\mathbf{0.0000}$ &$\mathbf{0.0000}$ &$\underline{0.0066}$\\
    &ProjFlow~\citep{watanabe2026projflow} &\ding{51} &$0.095$ &$0.771$ &$10.644$ &$\underline{0.0609}$ &$\mathbf{0.0000}$ &$\mathbf{0.0000}$ &$\mathbf{0.0000}$\\
    &\cellcolor[gray]{0.90}\textbf{\method} &\cellcolor[gray]{0.90}\ding{51} &\cellcolor[gray]{0.90}$\mathbf{0.041}$ &\cellcolor[gray]{0.90}$\mathbf{0.817}$ &\cellcolor[gray]{0.90}$\underline{10.630}$ &\cellcolor[gray]{0.90}$0.0639$ &\cellcolor[gray]{0.90}$\mathbf{0.0000}$ &\cellcolor[gray]{0.90}$\mathbf{0.0000}$ &\cellcolor[gray]{0.90}$\mathbf{0.0000}$\\
    \midrule

    \multirow{5}{*}{\textbf{Right foot}}
    &OmniControl~\citep{xie2024omnicontrol} &\ding{55} &$0.137$ &$\underline{0.798}$ &$\underline{10.241}$ &$0.0668$&$\underline{0.0666}$ &$\underline{0.0120}$ &$0.0334$\\
    &MotionLCM V2+CtrlNet~\citep{dai2024motionlcm} &\ding{55} &$4.756$ &$0.705$ &$9.303$ &$0.1026$&$0.2459$ &$0.1127$ &$0.2278$\\
    &MaskControl~\citep{pinyoanuntapong2025maskcontrol} &\ding{55} &$\underline{0.080}$ &$0.793$ &$10.159$ &$\mathbf{0.0552}$ &$\mathbf{0.0000}$ &$\mathbf{0.0000}$ &$\underline{0.0062}$\\
    &ProjFlow~\citep{watanabe2026projflow} &\ding{51} &$0.096$ &$0.770$ &$\mathbf{10.651}$ &$\underline{0.0613}$ &$\mathbf{0.0000}$ &$\mathbf{0.0000}$ &$\mathbf{0.0000}$\\
    &\cellcolor[gray]{0.90}\textbf{\method} &\cellcolor[gray]{0.90}\ding{51} &\cellcolor[gray]{0.90}$\mathbf{0.043}$ &\cellcolor[gray]{0.90}$\mathbf{0.818}$ &\cellcolor[gray]{0.90}$10.729$ &\cellcolor[gray]{0.90}$0.0642$ &\cellcolor[gray]{0.90}$\mathbf{0.0000}$ &\cellcolor[gray]{0.90}$\mathbf{0.0000}$ &\cellcolor[gray]{0.90}$\mathbf{0.0000}$\\

    \midrule

    \multirow{5}{*}{\textbf{Head}}
    &OmniControl~\citep{xie2024omnicontrol} &\ding{55} &$0.146$ &$0.796$ &$\underline{10.239}$ &$0.0556$ &$\underline{0.0422}$ &$\underline{0.0079}$ &$0.0349$\\
    &MotionLCM V2+CtrlNet~\citep{dai2024motionlcm} &\ding{55} &$4.580$ &$0.715$ &$9.278$ &$0.1138$&$0.1971$ &$0.0977$ &$0.2136$\\
    &MaskControl~\citep{pinyoanuntapong2025maskcontrol} &\ding{55} &$\underline{0.090}$ &$\underline{0.797}$ &$10.131$ &$\underline{0.0531}$ &$\mathbf{0.0000}$ &$\mathbf{0.0000}$ &$\underline{0.0064}$\\
    &ProjFlow~\citep{watanabe2026projflow} &\ding{51} &$0.099$ &$0.788$ &$10.754$ &$0.0595$ &$\mathbf{0.0000}$ &$\mathbf{0.0000}$ &$\mathbf{0.0000}$\\
    &\cellcolor[gray]{0.90}\textbf{\method} &\cellcolor[gray]{0.90}\ding{51} &\cellcolor[gray]{0.90}$\mathbf{0.066}$ &\cellcolor[gray]{0.90}$\mathbf{0.821}$ &\cellcolor[gray]{0.90}$\mathbf{10.642}$ &\cellcolor[gray]{0.90}$\mathbf{0.0528}$ &\cellcolor[gray]{0.90}$\mathbf{0.0000}$ &\cellcolor[gray]{0.90}$\mathbf{0.0000}$ &\cellcolor[gray]{0.90}$\mathbf{0.0000}$\\
    \midrule

    \multirow{5}{*}{\textbf{Left wrist}}
    &OmniControl~\citep{xie2024omnicontrol} &\ding{55} &$0.119$ &$0.783$ &$10.217$ & $0.0562$&$\underline{0.0801}$ &$\underline{0.0134}$ &$0.0529$\\
    &MotionLCM V2+CtrlNet~\citep{dai2024motionlcm} &\ding{55} &$4.103$ &$0.726$ &$9.188$ &$0.1167$ &$0.3965$ &$0.1912$ &$0.3150$\\
    &MaskControl~\citep{pinyoanuntapong2025maskcontrol} &\ding{55} &$0.118$ &$\underline{0.797}$ &$10.153$ &$\underline{0.0546}$ &$\mathbf{0.0000}$ &$\mathbf{0.0000}$ &$\underline{0.0044}$\\
    &ProjFlow~\citep{watanabe2026projflow} &\ding{51} &$\underline{0.089}$ &$0.783$ &$\underline{10.601}$ &$0.0586$ &$\mathbf{0.0000}$ &$\mathbf{0.0000}$ &$\mathbf{0.0000}$\\
    &\cellcolor[gray]{0.90}\textbf{\method} &\cellcolor[gray]{0.90}\ding{51} &\cellcolor[gray]{0.90}$\mathbf{0.075}$ &\cellcolor[gray]{0.90}$\mathbf{0.818}$ &\cellcolor[gray]{0.90}$\mathbf{10.576}$ &\cellcolor[gray]{0.90}$\mathbf{0.0535}$ &\cellcolor[gray]{0.90}$\mathbf{0.0000}$ &\cellcolor[gray]{0.90}$\mathbf{0.0000}$ &\cellcolor[gray]{0.90}$\mathbf{0.0000}$\\
    \midrule
    
    \multirow{5}{*}{\textbf{Right wrist}}
    &OmniControl~\citep{xie2024omnicontrol} &\ding{55} &$0.128$ &$0.792$ &$\underline{10.309}$ & $0.0601$&$\underline{0.0813}$ &$\underline{0.0127}$ &$0.0519$\\
    &MotionLCM V2+CtrlNet~\citep{dai2024motionlcm} &\ding{55} &$4.051$ &$0.725$ &$9.242$ &$0.1176$ &$ 0.3822$ &$0.1806$ &$0.3079$\\
    &MaskControl~\citep{pinyoanuntapong2025maskcontrol} &\ding{55} &$0.121$ &$\underline{0.797}$ &$10.105$ &$\underline{0.0537}$ &$\mathbf{0.0000}$ &$\mathbf{0.0000}$ &$\underline{0.0044}$\\
    &ProjFlow~\citep{watanabe2026projflow} &\ding{51} &$\underline{0.096}$ &$0.780$ &$10.610$ &$0.0584$ &$\mathbf{0.0000}$ &$\mathbf{0.0000}$ &$\mathbf{0.0000}$\\
    &\cellcolor[gray]{0.90}\textbf{\method} &\cellcolor[gray]{0.90}\ding{51} &\cellcolor[gray]{0.90}$\mathbf{0.075}$ &\cellcolor[gray]{0.90}$\mathbf{0.815}$ &\cellcolor[gray]{0.90}$\mathbf{10.533}$ &\cellcolor[gray]{0.90}$\mathbf{0.0534}$ &\cellcolor[gray]{0.90}$\mathbf{0.0000}$ &\cellcolor[gray]{0.90}$\mathbf{0.0000}$ &\cellcolor[gray]{0.90}$\mathbf{0.0000}$\\

    \midrule

    \multirow{5}{*}{\textbf{Average}}
    &OmniControl~\citep{xie2024omnicontrol} &\ding{55} &$0.126$ &$0.792$ &$\underline{10.276}$ & $0.0608$&$0.0617$ &$\underline{0.0107}$ &$0.0404$\\
    &MotionLCM V2+CtrlNet~\citep{dai2024motionlcm} &\ding{55} &$4.504$ &$0.715$ &$9.230$ &0.1119 &$0.2740$ &$0.1315$ &$0.2464$\\
    &MaskControl~\citep{pinyoanuntapong2025maskcontrol} &\ding{55} &$\underline{0.095}$ &$\underline{0.795}$ &$10.159$ &$\mathbf{0.0545}$ &$\underline{0.0001}$ &$\mathbf{0.0000}$ &$\underline{0.0065}$\\
    &ProjFlow~\citep{watanabe2026projflow} &\ding{51} &$0.097$ &$0.779$ &$10.651$ &$0.0603$ &$\mathbf{0.0000}$ &$\mathbf{0.0000}$ &$\mathbf{0.0000}$\\
    &\cellcolor[gray]{0.90}\textbf{\method} &\cellcolor[gray]{0.90}\ding{51} &\cellcolor[gray]{0.90}$\mathbf{0.061}$ &\cellcolor[gray]{0.90}$\mathbf{0.818}$ &\cellcolor[gray]{0.90}$\mathbf{10.593}$ &\cellcolor[gray]{0.90}$\underline{0.0578}$ &\cellcolor[gray]{0.90}$\mathbf{0.0000}$ &\cellcolor[gray]{0.90}$\mathbf{0.0000}$ &\cellcolor[gray]{0.90}$\mathbf{0.0000}$\\

    \midrule
    \midrule
    & \multirow{2}{*}{Methods} & \multirow{2}{*}{Zero-shot?} & \multirow{2}{*}{FID$\downarrow$} & R-Precision & R-Precision & R-Precision & \multirow{2}{*}{Matching$\downarrow$} & \multirow{2}{*}{Diversity$\rightarrow$} & \multirow{2}{*}{$-$}\\
    ~ & && & Top 1 & Top 2 & Top 3 & & & \\
    \midrule
    \multirow{7}{*}{\parbox{2.2cm}{\centering \textbf{Upper-Body\\Edit}}}
    &MDM~\citep{mdm2022human} &\ding{51} &$1.918$ &$0.359$ &$0.556$ &$0.654$ &$4.793$ &$9.210$ &$-$\\
    &OmniControl~\citep{xie2024omnicontrol} &\ding{55} &$0.909$ &$0.428$ &$0.614$ &$0.722$ &$3.694$ &$10.207$ &$-$\\
    &MotionLCM V2+CtrlNet~\citep{dai2024motionlcm} &\ding{55} &$3.922$ &$0.404$ &$0.592$ &$0.692$ &$5.610$ &$9.309$ &$-$\\
    &MaskControl~\citep{pinyoanuntapong2025maskcontrol} &\ding{55} &$\underline{0.066}$ &$0.501$ &$0.695$ &$\underline{0.794}$ &$3.227$ &$10.159$ &$-$\\
    &ACMDM-S-PS22+CtrlNet~\citep{meng2025absolute} &\ding{55} &$0.076$ &$\mathbf{0.532}$ &$\mathbf{0.719}$ &$\mathbf{0.820}$ &$\mathbf{3.098}$ &$\underline{10.586}$ &$-$\\
    &ProjFlow~\citep{watanabe2026projflow} &\ding{51} &$0.087$ &$0.501$ &$0.690$ &$0.787$ &$3.319$ &$\mathbf{10.571}$ &$-$\\
    &\cellcolor[gray]{0.90}\textbf{\method} &\cellcolor[gray]{0.90}\ding{51} &\cellcolor[gray]{0.90}$\mathbf{0.034}$ &\cellcolor[gray]{0.90}$\underline{0.511}$ &\cellcolor[gray]{0.90}$\underline{0.701}$ &\cellcolor[gray]{0.90}$\underline{0.794}$ &\cellcolor[gray]{0.90}$\underline{3.214}$ &\cellcolor[gray]{0.90}$10.714$ &\cellcolor[gray]{0.90}$-$\\
    \bottomrule
    \end{tabular}}
    \label{tab:supp_control_full}
\end{table}

The main paper reports the pelvis-only and all-joints-average results in \Tref{tab:control_noAITS}. \Tref{tab:supp_control_full} gives the complete per-joint evaluation under the same OmniControl~\citep{xie2024omnicontrol} protocol. Across the pelvis, feet, head, and wrists, \method has zero trajectory, location, and average errors. Compared with ProjFlow~\citep{watanabe2026projflow}, it lowers FID and raises R-Precision@3 for every reported joint. Averaged over the all-joints setting, FID improves from $0.097$ to $0.061$ and R-Precision@3 from $0.779$ to $0.818$. The average foot-skating ratio also decreases from $0.0603$ to $0.0578$, although exact satisfaction of projected coordinates and a single physical metric do not establish the quality of the unconstrained motion.

Following MaskControl~\citep{pinyoanuntapong2025maskcontrol} and ProjFlow~\citep{watanabe2026projflow}, upper-body editing conditions on the ground-truth pelvis, left-foot, and right-foot signals at every frame and generates the remaining motion from the text description. The final block reports this task using the comparison and metric format of ProjFlow. \method achieves the best FID, reducing it from $0.066$ to $0.034$ relative to the strongest baseline, and obtains R-Precision of $0.511$, $0.701$, and $0.794$ at ranks 1--3, a matching distance of $3.214$, and diversity of $10.714$.

\section{Implementation and Reproducibility Details}
\label{app:repro}

We next specify the two principal \method variants and give the training, sampling, and endpoint-projection algorithms corresponding to the framework in \Fref{fig:framework}. Unless otherwise stated, models are trained for 500 epochs with batch size 64 using AdamW~\citep{loshchilov2019decoupled} and learning rate $2\times10^{-4}$. The 68M-parameter transformer has eight RA-MMDiT blocks of width 512, four attention heads, and feed-forward width 1024. We use text dropout $0.1$, classifier-free guidance weight 3, one sequence-level $t\sim\mathcal U(0,1)$ per training example, and $\sigma_{\min}=0.05$ for training and $0.01$ for inference. Sampling uses EMA weights, fixed-step Heun updates, and a final Euler step. A complete 500-epoch training run of \method takes approximately 450 minutes on one NVIDIA A100-SXM4-80GB GPU.

\subsection{Causal 263D Text-to-Motion Configuration}
\label{app:config-263d}

\Tref{tab:config} summarizes the configuration of the causal 263D generator.

\begin{table}[t]
\centering
\caption{Configuration of the causal 263D model used for the primary HumanML3D result. The evaluated EMA checkpoint was not retained with an immutable run and epoch identifier.}
\label{tab:config}
\begin{tabular}{ll}
\toprule
Item & Selected value \\
\midrule
Model variant & causal 263D, $s=5$ \\
Training representation & HumanML3D, 263 dimensions \\
Evaluation representation & root/RIC subset, 67 dimensions \\
Backbone & 8 joint blocks, width 512, 4 heads \\
Motion block size & 1 frame \\
Training / inference mask & causal / causal \\
Flow path / time & linear / uniform sequence-level $t$ \\
Network output / loss & endpoint $\hat\x_1$ / stabilized residual MSE \\
Source distribution & $\mathcal N(0,25\mathbf I)$ \\
Text encoder & frozen DistilBERT, token features \\
Token Refiner & 2 layers, timestep/context conditioned \\
CFG / condition dropout & 3.0 / 0.1 \\
Sampler & 50 time points, Heun + final Euler \\
Evaluated weights & latest EMA checkpoint \\
\bottomrule
\end{tabular}
\end{table}

\begin{table}[t]
\centering
\caption{Configuration of the bidirectional XYZ model used for the primary HumanML3D result.}
\label{tab:config-xyz}
\begin{tabular}{ll}
\toprule
Item & Selected value \\
\midrule
Model variant & bidirectional XYZ, $s=5$ \\
Training representation & shared-axis-normalized XYZ, 66 dimensions \\
Evaluation representation & root/RIC subset, 67 dimensions \\
Backbone & 8 joint blocks, width 512, 4 heads \\
Motion block size & 1 frame \\
Training / inference mask & bidirectional / bidirectional \\
Flow path / time & linear / uniform sequence-level $t$ \\
Network output / loss & endpoint $\hat\x_1$ / stabilized residual MSE \\
Source distribution & $\mathcal N(0,25\mathbf I)$ \\
Text encoder & frozen DistilBERT, token features \\
Token Refiner & 2 layers, timestep/context conditioned \\
CFG / condition dropout & 3.0 / 0.1 \\
Sampler & 50 time points, Heun + final Euler \\
Evaluated weights & latest EMA checkpoint labeled \\
\bottomrule
\end{tabular}
\end{table}

\subsection{Bidirectional XYZ Text-to-Motion Configuration}
\label{app:config-xyz}

\Tref{tab:config-xyz} summarizes the corresponding configuration of the bidirectional XYZ generator.

\FloatBarrier

\subsection{Training and Sampling Algorithms}
\label{app:training-sampling}

Algorithms~\ref{alg:training} and~\ref{alg:sampling} detail the shared training and text-to-motion sampling procedures, respectively.

\begin{algorithm}[t]
\caption{Representation-aware direct flow training}
\label{alg:training}
\begin{algorithmic}[1]
\Require Motion--text minibatch $(\x_1,\mathbf c)$, source scale $s$
\State Sample $\boldsymbol\epsilon\sim\mathcal N(0,\mathbf I)$ and $t\sim\mathcal U(0,1)$
\State Form $\x_0\gets s\boldsymbol\epsilon$ and $\x_t\gets(1-t)\x_0+t\x_1$
\State Refine frozen text tokens using $t$ and their masked global mean
\State Choose the fixed mask $\mathbf M_{\operatorname{Att}}$ (causal 263D or bidirectional XYZ)
\State Predict $\hat\x_1\gets f_\theta(\x_t,t,\mathbf c;\mathbf M_{\operatorname{Att}})$
\State Convert $\hat\x_1$ and $\x_1$ to stabilized residual fields with \Eref{eq:x-to-v}
\State Update $\theta$ with valid-frame residual MSE; update EMA weights
\end{algorithmic}
\end{algorithm}

\begin{algorithm}[t]
\caption{\method sampling}
\label{alg:sampling}
\begin{algorithmic}[1]
\Require Text $\mathbf c$, length $T$, source scale $s$, guidance $w$
\State Sample the full sequence $\x_0\sim\mathcal N(0,s^2\mathbf I)$
\State Encode and refine conditional and null text features
\For{successive time points $t_k,t_{k+1}$}
  \State Predict the guided endpoint with classifier-free guidance using the checkpoint mask $\mathbf M_{\operatorname{Att}}$
  \State Convert the endpoint to the residual field with \Eref{eq:x-to-v}
  \State Update all motion frames with one Heun step
\EndFor
\State Replace the final corrector with one Euler update
\State \Return generated motion features $\x_1$
\end{algorithmic}
\end{algorithm}

\begin{table}[t]
\centering
\caption{Inference configuration for bidirectional XYZ endpoint projection. The model is trained for text-to-motion generation without masks or target coordinates.}
\label{tab:rct-control-config}
\begin{tabular}{ll}
\toprule
Item & Control-evaluation setting \\
\midrule
Training attention & bidirectional, fixed for the XYZ checkpoint \\
Inference attention & bidirectional \\
Control training & none \\
Raw controls & any valid frames, selected joint IDs, and XYZ axes \\
CFG weight / projected steps & 3.0 / 100 \\
Source noise & scale $s$ of the training path \\
Noise refresh & source-scale corrected \\
\bottomrule
\end{tabular}
\end{table}

\begin{algorithm}[t]
\caption{Bidirectional zero-shot XYZ control}
\label{alg:rct-control}
\begin{algorithmic}[1]
\Require Bidirectional XYZ checkpoint, text $\mathbf c$, length $T$, target $\y$, mask $\mathbf M$
\State Resolve source scale $s$ from the checkpoint
\State Sample $\x_0\sim\mathcal N(\mathbf0,s^2\mathbf I)$
\For{successive time points $t,t'$}
  \State Predict the CFG velocity and recover $(\hat\x_1,\hat\x_0)$
  \State Project $\hat\x_1$ onto $(\mathbf M,\y)$ using \Eref{eq:control-projection}
  \State Optionally refresh $\hat\x_0$ using \Eref{eq:control-refresh}
  \State Recompose $\x_{t'}$ using \Eref{eq:control-recompose}
\EndFor
\State Hard-replace controlled coordinates and zero padded frames
\State \Return controlled direct-space motion $\x_1$
\end{algorithmic}
\end{algorithm}

\subsection{Bidirectional XYZ Endpoint-Projection Configuration}
\label{app:rct-control}

Spatial control reuses the source scale and attention graph of the trained bidirectional XYZ generator. Arbitrary binary masks select valid XYZ entries, while neither masks nor target coordinates are provided during training. \Tref{tab:rct-control-config} summarizes this inference-only configuration, and Algorithm~\ref{alg:rct-control} details the endpoint-projection sampler.

\section{Sample Code}
\label{app:sample-code}

We provide code for training and evaluating the proposed \method on the HumanML3D dataset. \textbf{\textit{Please refer to the sample code on our project website for details.}}

\section{Theoretical Details and Proofs}
\label{app:derivations}

This section provides the proofs underlying the three design questions in the main paper: the autoencoder fidelity floor in latent generation, the geometry induced by a scaled source, and the relationship between temporal attention and motion representation. For the last question, we also provide the full statistical formulation and conclude with the dynamical prefix property of the causal vector field.

\subsection{Proofs of Main Propositions}
\label{app:main-proofs}

Proposition~\ref{prop:decoder-floor} formalizes the irreducible fidelity error introduced when generated motions are restricted to a fixed decoder's range, motivating direct motion-space generation. Proposition~\ref{prop:source-geometry} explains how the Gaussian source scale controls both when motion signal emerges along the flow path and the conditioning of intermediate distributions.

\begin{proof}[Proof of Proposition~\ref{prop:decoder-floor}]
Let $\pi$ be any coupling of $\mathbf X\sim P_1$ and $\mathbf Y\sim Q$. Because $Q$ is supported on $\mathcal R_g$, $\mathbf Y\in\mathcal R_g$ almost surely. Pointwise,
\begin{equation}
 \|\mathbf X-\mathbf Y\|_2^2
 \geq \inf_{\mathbf y\in\mathcal R_g}\|\mathbf X-\mathbf y\|_2^2
 =\operatorname{dist}^2(\mathbf X,\mathcal R_g).
\end{equation}
Taking the expectation under $\pi$ and then the infimum over all couplings yields \Eref{eq:decoder-floor}.
\end{proof}

\begin{proof}[Proof of Proposition~\ref{prop:source-geometry}]
Project \Eref{eq:linear-path} onto an eigenvector $\mathbf u$: $\mathbf u^\top\mathbf X_t=t\mathbf u^\top\mathbf X_1+(1-t)s\mathbf u^\top\boldsymbol\epsilon$. Independence gives signal variance $t^2\lambda_{\mathbf u}$ and source-noise variance $(1-t)^2s^2$, proving \Eref{eq:path-snr}. The inequality $\operatorname{SNR}_{\mathbf u}(t;s)\leq1$ is equivalent on $[0,1]$ to $t\leq s/(s+\sqrt{\lambda_{\mathbf u}})$. Under uniform time sampling, the probability is the length of this interval.

Independence also gives \Eref{eq:path-covariance}. Its extreme eigenvalues are $t^2\lambda_{\min}+(1-t)^2s^2$ and $t^2\lambda_{\max}+(1-t)^2s^2$, yielding \Eref{eq:path-condition}. Let $a=t^2$ and $b=(1-t)^2s^2$. For $t<1$,
\begin{equation}
 \frac{\partial}{\partial b}
 \frac{a\lambda_{\max}+b}{a\lambda_{\min}+b}
 =\frac{a(\lambda_{\min}-\lambda_{\max})}
 {(a\lambda_{\min}+b)^2}\leq0.
\end{equation}
Since $b$ is increasing in $s>0$, $\kappa_t(s)$ is non-increasing.
\end{proof}

\subsection{Empirical Spatiotemporal Covariance and Source Scale}
\label{app:motion-correlation}

We complement Proposition~\ref{prop:source-geometry} with a full-split covariance analysis of the two training representations. For each window length $T_0\in\{32,64,128\}$, we draw one uniformly positioned window from every eligible HumanML3D training motion and flatten it to a vector $\mathbf x_n\in\mathbb R^p$, where $p=T_0D$. Each coordinate of $\mathbf x_n$ therefore identifies one feature channel $d$ at one frame $\tau$. A covariance entry indexed by $(\tau,d)$ and $(\tau',d')$ measures how those two variables vary together across motions. It includes same-frame relationships between channels, temporal relationships between frames, and cross-channel relationships across different frames; we therefore call it \emph{spatiotemporal} covariance. For XYZ, the channels are joint coordinates, while for 263D they are the channels of the HumanML3D representation.

Sampling at most one window per motion avoids overweighting long or highly correlated sequences. The filter $\max(40,T_0)\leq T<200$ yields 22,326, 20,850, and 13,328 motions for $T_0=32$, 64, and 128, respectively. The 263D windows use the official feature-wise HumanML3D normalization, while XYZ uses the shared-axis normalization used to train the global generator.

We first form the empirical covariance across the $N$ sampled motion windows. Ledoit--Wolf shrinkage then moves this estimate toward a scaled identity matrix to stabilize its spectrum when the dimension $p$ is large relative to $N$. We compare every result with an isotropic Gaussian matrix processed by the identical estimator and matched in both $N$ and $p$. For correlation matrix $\widehat{\mathbf R}$ and shrinkage-covariance eigenvalues $\lambda_k$, we report
\begin{align}
    C_{\mathrm{off}}
    &=\left(\frac{1}{p(p-1)}\sum_{i\ne j}\widehat R_{ij}^{\,2}\right)^{1/2},
    & A_{\mathrm{spec}}
    &=\frac{\lambda_{\max}}{p^{-1}\sum_k\lambda_k},\\
    \overline r_{\mathrm{eff}}
    &=\frac{1}{p}\exp\!\left(-\sum_k q_k\log q_k\right),
    & q_k&=\frac{\lambda_k}{\sum_j\lambda_j}.
    \label{eq:covariance-diagnostics}
\end{align}
Here, $C_{\mathrm{off}}$ is the root-mean-square off-diagonal correlation. It summarizes pairwise dependence between distinct frame--channel variables: zero means no linear pairwise correlation, and larger values indicate stronger dependence. $A_{\mathrm{spec}}$ divides the largest covariance eigenvalue by the mean eigenvalue. It equals one when every direction has equal variance and grows when one direction dominates the covariance spectrum.

The normalized weight $q_k$ is the fraction of total covariance variance assigned to eigen-direction $k$, so $q_k\geq0$ and $\sum_k q_k=1$. Exponentiating the entropy of these weights gives the effective number of active covariance directions; division by $p$ produces $\overline r_{\mathrm{eff}}\in(0,1]$. A value near one means variance is spread evenly across the available directions, whereas a value near zero means it is concentrated in relatively few directions. Thus, an isotropic population has $C_{\mathrm{off}}=0$, $A_{\mathrm{spec}}=1$, and $\overline r_{\mathrm{eff}}=1$. Effective rank describes the covariance spectrum and is not an estimate of the nonlinear intrinsic dimension of the motion manifold.

\begin{table}[t]
\centering
\caption{\textbf{Spatiotemporal covariance diagnostics across window lengths.}
Each statistic is reported as motion / a matched isotropic Gaussian with the
same sample count $N$, dimension $p$, and covariance estimator. Lower
$C_{\mathrm{off}}$, lower $A_{\mathrm{spec}}$, and higher
$\overline r_{\mathrm{eff}}$ indicate greater isotropy.}
\label{tab:supp-motion-correlations}
\setlength{\tabcolsep}{5pt}
\resizebox{0.7\linewidth}{!}{
\begin{tabular}{lrrrrr}
\toprule
Representation & $T_0$ & $N$ & $C_{\mathrm{off}}$ & $A_{\mathrm{spec}}$ & $\overline r_{\mathrm{eff}}$ \\
\midrule
\multirow{3}{*}{263D}
 & 32  & 22,326 & $.166/.007$ & $1088.4/1.0$ & $.009/1.000$ \\
 & 64  & 20,850 & $.148/.007$ & $2116.0/1.0$ & $.007/1.000$ \\
 & 128 & 13,328 & $.137/.009$ & $4336.5/1.0$ & $.004/1.000$ \\
\midrule
\multirow{3}{*}{XYZ}
 & 32  & 22,326 & $.485/.007$ & $1027.0/1.0$ & $.002/1.000$ \\
 & 64  & 20,850 & $.444/.007$ & $1808.7/1.0$ & $.001/1.000$ \\
 & 128 & 13,328 & $.390/.009$ & $3346.3/1.0$ & $.001/1.000$ \\
\bottomrule
\end{tabular}}
\end{table}

\begin{table}[t]
\centering
\caption{\textbf{Covariance diagnostics of pretrained latent representations.}
All methods use the same 20,850 HumanML3D training motions and one 64-frame
window per motion. Latent shape excludes the batch dimension. Each diagnostic
is reported as latent / a matched isotropic Gaussian with the same $N$, $p$,
and covariance estimator.}
\label{tab:supp-latent-correlations}
\setlength{\tabcolsep}{5pt}
\begin{tabular}{llrrr}
\toprule
Method & Latent shape & $C_{\mathrm{off}}$ & $A_{\mathrm{spec}}$ & $\overline r_{\mathrm{eff}}$ \\
\midrule
CMDM  & $16\times64$          & $.181/.007$ & $118.9/1.0$ & $.077/1.000$ \\
SALAD & $16\times7\times32$  & $.076/.007$ & $129.3/1.0$ & $.238/1.000$ \\
\bottomrule
\end{tabular}
\end{table}

\Tref{tab:supp-motion-correlations} shows substantial spatiotemporal dependence after normalization. Across all window lengths, motion has much larger residual correlation and spectral anisotropy than its matched Gaussian, together with a far smaller normalized effective rank. For example, at $T_0=64$, 263D motion has $C_{\mathrm{off}}=.148$, $A_{\mathrm{spec}}=2116.0$, and $\overline r_{\mathrm{eff}}=.007$, whereas the matched Gaussian gives $.007$, $1.0$, and $1.000$. XYZ is even more correlated and spectrally concentrated, with $.444$, $1808.7$, and $.001$, respectively. The same separation holds at 32 and 128 frames, directly verifying that marginal normalization does not make either motion space isotropic.

For comparison with latent-space generators, we apply the same estimator to the pretrained representations used by CMDM and SALAD. Both results use the common $T_0=64$ protocol and the same 20,850 motions, producing 16 latent frames for each method; we preserve each representation's native frame structure when flattening and use one seeded posterior sample per motion, matching the variables seen during generator training.

\Tref{tab:supp-latent-correlations} reports only the covariance diagnostics for CMDM and SALAD; we compare them here with the direct-motion results in \Tref{tab:supp-motion-correlations}. SALAD is the most isotropic representation by both $C_{\mathrm{off}}=.076$ and $\overline r_{\mathrm{eff}}=.238$. CMDM has slightly higher pairwise correlation than 263D ($.181$ versus $.148$), but its much smaller $A_{\mathrm{spec}}$ ($118.9$ versus $2116.0$) and larger $\overline r_{\mathrm{eff}}$ ($.077$ versus $.007$) show that variance is distributed over substantially more directions. These CMDM and SALAD results are consistent with variational regularization, although both remain far from their matched Gaussian references. Because dimensionality and preprocessing differ across representations, these statistics characterize covariance geometry rather than rank generators by sample quality. Nevertheless, they reveal clear differences between latent-space and motion-space representations.

We next substitute the extreme eigenvalues of each estimated shrinkage covariance into \Eref{eq:path-condition}. This quantifies the conditioning of the corresponding spatiotemporal probability path without Monte Carlo noise. We report $\log_{10}\kappa_t(s)$ at $T_0=64$ for an early path time $t=0.1$, the midpoint $t=0.5$, and a late path time $t=0.9$.

\begin{table}[t]
\centering
\caption{\textbf{Source scale changes path conditioning and generation
quality.} The columns headed by $t$ report $\log_{10}\kappa_t(s)$, computed at
$T_0=64$ from the extreme eigenvalues of the estimated shrinkage covariance;
lower is better. FID uses the matched clean-prediction configurations
from~\Tref{tab:design-analysis}: causal attention for 263D and bidirectional
attention for XYZ.}
\label{tab:supp-source-scale}
\setlength{\tabcolsep}{5pt}
\begin{tabular}{lrrrrr}
\toprule
Representation & Scale $s$ & $t=0.1$ & $t=0.5$ & $t=0.9$ & FID$\downarrow$ \\
\midrule
\multirow{2}{*}{263D} & 1 & $1.516$ & $3.410$ & $5.241$ & $.111^{\pm.005}$ \\
 & 5 & $\mathbf{.356}$ & $\mathbf{2.017}$ & $\mathbf{3.918}$ & $\mathbf{.046}^{\pm.004}$ \\
\midrule
\multirow{2}{*}{XYZ} & 1 & $1.253$ & $3.136$ & $5.023$ & $.144^{\pm.009}$ \\
 & 5 & $\mathbf{.224}$ & $\mathbf{1.746}$ & $\mathbf{3.646}$ & $\mathbf{.038}^{\pm.004}$ \\
\bottomrule
\end{tabular}
\end{table}

Increasing $s$ from 1 to 5 reduces $\log_{10}\kappa_t$ from $1.516/3.410/5.241$ to $.356/2.017/3.918$ at $t=0.1/0.5/0.9$ for 263D (\Tref{tab:supp-source-scale}). For XYZ, the corresponding values fall from $1.253/3.136/5.023$ to $.224/1.746/3.646$. Thus, $s$ materially changes the conditioning of the correlated motion path at early, middle, and late times; it is not merely an inference-time diversity parameter. The matched clean-prediction ablations exhibit correspondingly large sensitivity: FID falls by $58.6\%$ for causal 263D and by $73.6\%$ for bidirectional XYZ when moving from $s=1$ to $s=5$. These controlled results establish the importance of source scale for the selected motion-space generators, while the covariance analysis verifies its path-level conditioning effect. They do not imply that improved conditioning alone causes the FID gains or that $s=5$ is universally optimal: the best scale can change with the representation, attention graph, and prediction parameterization.

\subsection{Why Attention Depends on the Representation}
\label{app:attention-representation}

The attention choice reflects what a frame token represents. Incremental features describe local changes whose global effect is obtained by forward accumulation, suggesting a filtering-style causal dependency. Absolute coordinates describe a globally coupled trajectory for which future observations can help smooth an earlier state. We formalize this distinction and its limitations below.

Let $\operatorname{Att}(i)$ be the set of motion-key frames visible to the prediction at frame $i$:
\begin{equation}
    \operatorname{Att}_{\mathrm C}(i)=\{1,\ldots,i\},
    \qquad
    \operatorname{Att}_{\mathrm B}(i)=\{1,\ldots,T\}.
    \label{eq:attention-sets}
\end{equation}
The corresponding field component has the dependency
\begin{equation}
    \hat{\mathbf v}_{\theta,i}(\x_t,t,\mathbf c)
    =\hat{\mathbf v}_{\theta,i}(\x_{t,\operatorname{Att}(i)},t,\mathbf c).
    \label{eq:masked-field}
\end{equation}

To isolate the statistical cost of the causal restriction, set $Y_i=\mathbf X_{1,i}$, $\mathcal G_i=\sigma(\mathbf X_{t,1:i},t,\mathbf c)$, and $\mathcal H=\sigma(\mathbf X_{t,1:T},t,\mathbf c)$. Let $R_{\mathrm C,i}$ and $R_{\mathrm B,i}$ be the mean-squared errors of the corresponding Bayes-optimal conditional means.

\begin{proposition}[Causal excess prediction risk]
\label{prop:attention-risk}
For square-integrable motion,
\begin{equation}
    R_{\mathrm C,i}-R_{\mathrm B,i}
    =\mathbb E\left[
    \left\|
    \mathbb E[Y_i\mid\mathcal H]
    -\mathbb E[Y_i\mid\mathcal G_i]
    \right\|_2^2\right]\geq0.
    \label{eq:attention-risk}
\end{equation}
Equality holds exactly when the full-sequence conditional mean is already measurable from the noisy prefix.
\end{proposition}

\begin{proof}[Proof of Proposition~\ref{prop:attention-risk}]
Because $\mathcal G_i\subseteq\mathcal H$, conditional expectation is an orthogonal projection in $L^2$. Add and subtract $\mathbb E[Y_i\mid\mathcal H]$ inside the causal residual:
\begin{align}
Y_i-\mathbb E[Y_i\mid\mathcal G_i]
={}&Y_i-\mathbb E[Y_i\mid\mathcal H]\\
&+\mathbb E[Y_i\mid\mathcal H]-\mathbb E[Y_i\mid\mathcal G_i].
\end{align}
The first term is orthogonal to every $\mathcal H$-measurable random variable, including the second term. Taking squared norms and expectations gives
\begin{equation}
R_{\mathrm C,i}=R_{\mathrm B,i}
+\mathbb E\left\|
\mathbb E[Y_i\mid\mathcal H]-\mathbb E[Y_i\mid\mathcal G_i]
\right\|_2^2,
\end{equation}
which proves \Eref{eq:attention-risk}. Equality holds if and only if the two conditional means agree almost surely.
\end{proof}

Additional context cannot increase Bayes-optimal squared error, but it need not provide useful information. The following idealized model identifies when the inequality is strict.

\begin{proposition}[Increment prediction versus position smoothing]
\label{prop:representation-model}
Let $V_k\sim\mathcal N(0,q)$ and $E_k\sim\mathcal N(0,1)$ be mutually independent, with $t>0$ and $\tau>0$. For an increment representation observed as $Z_k=tV_k+\tau E_k$, suffix observations $Z_{i+1:T}$ do not reduce the Bayes risk for $V_i$. For an absolute-position representation $P_k=\sum_{r\leq k}V_r$ observed as $Z_k=tP_k+\tau E_k$, observing $Z_{i+1}$ strictly reduces the Bayes risk for $P_i$ whenever its prefix posterior variance is nonzero.
\end{proposition}

\begin{proof}[Proof of Proposition~\ref{prop:representation-model}]
For the increment representation, let the noisy observation of frame $k$ be $Z_k=tV_k+\tau E_k$, where $\tau>0$ and the $E_k$ are independent standard Gaussians. Independence across $k$ gives $V_i\perp Z_{i+1:T}\mid Z_{1:i}$, hence $\mathbb E[V_i\mid Z_{1:T}]=\mathbb E[V_i\mid Z_{1:i}]$ and Proposition~\ref{prop:attention-risk} gives equal risks.

For the position representation, let $P_k=\sum_{r\leq k}V_r$ and $Z_k=tP_k+\tau E_k$. Conditional on the prefix observations, $Z_{i+1}=tP_i+tV_{i+1}+\tau E_{i+1}$. Since $V_{i+1}$ and $E_{i+1}$ are independent of $(P_i,Z_{1:i})$,
\begin{equation}
\operatorname{Cov}(P_i,Z_{i+1}\mid Z_{1:i})
=t\operatorname{Var}(P_i\mid Z_{1:i}).
\end{equation}
For $t>0$ and nonzero prefix posterior variance, this conditional covariance is nonzero. Gaussian conditioning therefore strictly reduces $\operatorname{Var}(P_i\mid Z_{1:i})$ when $Z_{i+1}$ is added. The full suffix contains $Z_{i+1}$, so bidirectional Bayes risk is strictly lower.
\end{proof}

This is the classical filtering--smoothing distinction. Importantly, Propositions~\ref{prop:attention-risk} and~\ref{prop:representation-model} do not claim that causal attention has a lower Bayes-optimal reconstruction error. Bidirectional context contains the causal context, so its optimal squared error can only be smaller; for independent increments the two risks are equal because the suffix is redundant. The empirical advantage of a causal mask must therefore come from its inductive bias in a finite model rather than from greater information. When suffix observations carry little task-relevant information, removing their edges reduces the dependency class, discourages a denoising shortcut based on symmetric temporal smoothing, and directs limited capacity toward prefix-conditioned transitions. Moreover, the flow-matching loss is a local regression objective, whereas FID and retrieval evaluate the distribution produced after repeatedly applying the learned field. A context pattern that eases framewise regression need not yield the best integrated generative dynamics.

HumanML3D is not a sequence of independent increments: it mixes root velocities, root-relative positions, rotations, local velocities, and contacts under a text condition. Its root trajectory nevertheless has a distinguished forward accumulation structure, while an individual frame already specifies a complete root-relative pose. This makes the idealized increment model a plausible explanation for why suffix motion can be redundant for important components, not a literal model of the entire representation. Conversely, absolute XYZ trajectories are temporally correlated observations of global joint positions, making their suffix genuinely informative. The proposition therefore supplies a mechanism, not a universal ranking of masks; the representation-specific preference remains an empirical question.

\subsection{Empirical Attention Routing}
\label{app:attention-routing}

We inspect eight models with a balanced $2\times2\times2$ comparison of 263D versus XYZ representations, causal versus bidirectional masks, and source scales $s\in\{1,5\}$. We use 512 HumanML3D test motions available in both representations, with identical captions and representation-specific Gaussian draws. We evaluate $t\in\{0.2,0.5,0.8\}$ and average within each motion before aggregating, so heads and layers are not treated as independent observations. Motion-direction statistics condition on the mass assigned to valid motion keys and therefore do not depend on prompt length or text-attention mass.

\begin{table*}[t]
\centering
\caption{\textbf{Midpoint attention routing for all eight selected checkpoints.} Representation specifies the motion input format, Scale is the Gaussian source scale $s$, and Mask is the temporal attention graph. The remaining columns report percentages at $t=0.5$, averaged over 512 paired test motions. Future is the fraction of motion-to-motion attention from query frame $i$ assigned to later key frames $j>i$; Local $\pm4$ is the fraction assigned within four frames of the query; and Span is the attention-weighted expected absolute offset $|j-i|$, normalized by sequence length. Text is the fraction of a motion query's total attention assigned to text keys, while Text$\rightarrow$motion is the fraction of a text query's total attention assigned to motion keys.}
\label{tab:supp-attention-routing}
\setlength{\tabcolsep}{5pt}
\resizebox{0.8\linewidth}{!}{
\begin{tabular}{lcl|rrrrr}
\toprule
Representation & Scale & Mask & Future & Local $\pm4$ & Span & Text & Text$\rightarrow$motion \\
\midrule
\multirow{4}{*}{263D} & \multirow{2}{*}{$s=1$} & Causal & $0.0$ & $43.6$ & $13.6$ & $48.9$ & $0.0$ \\
 & & Bidirectional & $47.4$ & $50.3$ & $9.7$ & $31.6$ & $76.4$ \\
 & \multirow{2}{*}{$s=5$} & Causal & $0.0$ & $46.6$ & $10.4$ & $27.2$ & $0.0$ \\
 & & Bidirectional & $49.1$ & $38.1$ & $12.7$ & $17.5$ & $86.4$ \\
\midrule
\multirow{4}{*}{XYZ} & \multirow{2}{*}{$s=1$} & Causal & $0.0$ & $24.2$ & $20.6$ & $61.6$ & $0.0$ \\
 & & Bidirectional & $49.6$ & $25.4$ & $17.5$ & $24.2$ & $86.1$ \\
 & \multirow{2}{*}{$s=5$} & Causal & $0.0$ & $24.5$ & $18.2$ & $52.0$ & $0.0$ \\
 & & Bidirectional & $51.4$ & $15.8$ & $20.9$ & $29.1$ & $76.7$ \\
\bottomrule
\end{tabular}}
\end{table*}

\begin{figure*}[t]
    \centering
    \includegraphics[width=\textwidth]{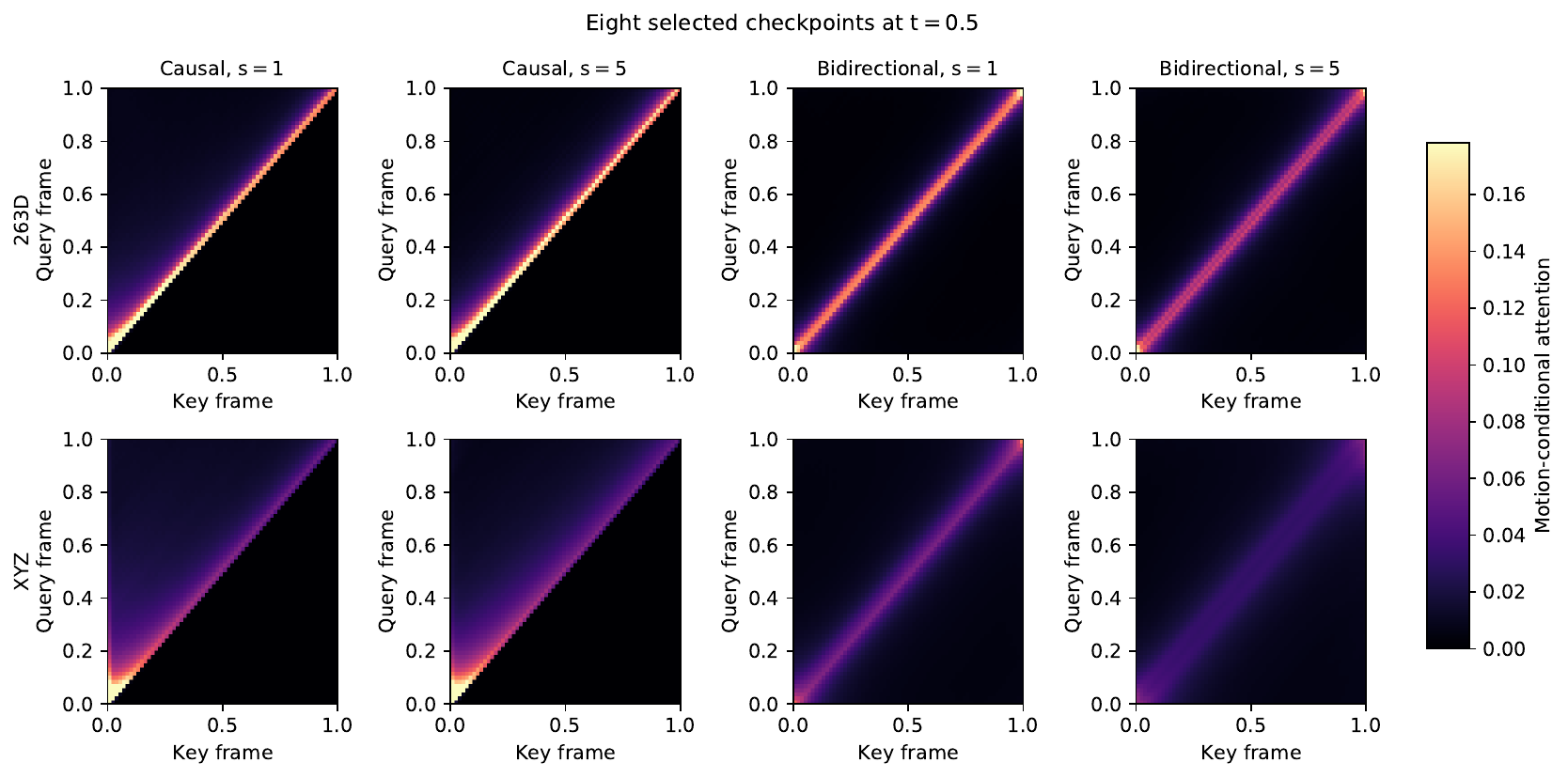}
    \caption{\textbf{Representation-dependent attention routing for all eight selected checkpoints.} Rows compare 263D and XYZ; columns compare causal and bidirectional masks at source scales $s=1$ and $s=5$. Each map averages 512 paired test motions, four heads, and eight layers at $t=0.5$ after conditioning each motion-query row on its mass assigned to valid motion keys. The causal mask removes the suffix exactly. Within the bidirectional graph, XYZ distributes attention over a broader temporal range than 263D at both source scales.}
    \label{fig:attention-routing}
\end{figure*}

\begin{figure*}[t]
    \centering
    \includegraphics[width=0.82\textwidth]{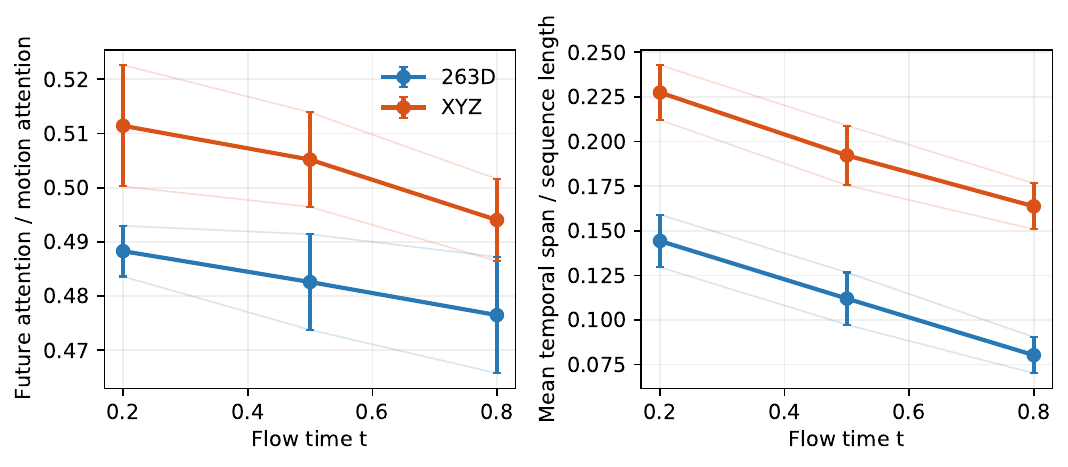}
    \caption{\textbf{Bidirectional routing across flow time for the selected checkpoints.} Each checkpoint statistic averages 512 paired test motions. Thin curves show the individual $s=1$ and $s=5$ models; markers and error bars show their mean and standard error. Future attention is the fraction of motion-to-motion attention from a query frame $i$ assigned to later key frames $j>i$; span is the attention-weighted expected absolute offset $|j-i|$, normalized by sequence length. XYZ has a longer normalized temporal span at every evaluated time, while both representations allocate approximately half of their motion attention to future frames.}
    \label{fig:attention-routing-time}
\end{figure*}

\Tref{tab:supp-attention-routing} reports the midpoint statistics for all eight checkpoints, and \Fref{fig:attention-routing} visualizes their representation-dependent routing patterns. The $s=5$ maps expose a learned difference within the same bidirectional graph: XYZ assigns only $15.8\%$ of motion attention within $\pm4$ frames, compared with $38.1\%$ for 263D, and its normalized span is $20.9\%$ rather than $12.7\%$. The independently trained $s=1$ pair shows the same ordering: XYZ is less local ($25.4\%$ versus $50.3\%$) and broader ($17.5\%$ versus $9.7\%$). Averaged over the two selected scales, bidirectional XYZ remains less local ($20.6\%$ versus $44.2\%$), has a longer span ($19.2\%$ versus $11.2\%$), and assigns a slightly larger fraction to future motion ($50.5\%$ versus $48.3\%$). As shown in~\Fref{fig:attention-routing-time}, the span ordering holds at every evaluated time: XYZ versus 263D is $22.7\%$ versus $14.4\%$ at $t=0.2$, $19.2\%$ versus $11.2\%$ at $t=0.5$, and $16.4\%$ versus $8.0\%$ at $t=0.8$.

These statistics are consistent with the proposed smoothing role of bidirectional XYZ attention: when suffix edges are available, both selected XYZ models use broader two-sided context rather than merely retaining a nominally dense mask. Conversely, all four selected causal checkpoints assign exactly zero weight to future motion and prevent text queries from reading motion, verifying the intended filtering graph and excluding an indirect suffix path through the updated text stream. Attention weights describe routing rather than causal attribution, however; values, output projections, residual paths, and subsequent layers also determine the influence on the generated motion.

Beyond the Gaussian analysis in Proposition~\ref{prop:representation-model}, Proposition~\ref{prop:prefix-autonomy} establishes a distribution-independent dynamical guarantee of the causal mask: the generated prefix is autonomous from perturbations to the source suffix.

\begin{proposition}[Prefix autonomy]
\label{prop:prefix-autonomy}
Assume the causal field is locally Lipschitz in motion state, so its sampling ODE has a unique solution. For any $k\leq T$, the generated trajectory $\x_{t,1:k}$ depends only on the source prefix $\x_{0,1:k}$ and text $\mathbf c$. Changing $\x_{0,k+1:T}$ cannot alter this prefix. Wherever the field is differentiable, its motion Jacobian is block lower triangular.
\end{proposition}

\begin{proof}[Proof of Proposition~\ref{prop:prefix-autonomy}]
For any $k$, causality implies that the first $k$ sampling equations can be written as the closed subsystem
\begin{equation}
 \frac{d\x_{t,1:k}}{dt}
 =\hat{\mathbf v}_{\theta,1:k}(\x_{t,1:k},t,\mathbf c).
\end{equation}
Its initial condition is $\x_{0,1:k}$. Local Lipschitz continuity gives uniqueness, so two full-sequence initial states with the same prefix must induce the same prefix trajectory even if their suffixes differ. If the field is differentiable, \Eref{eq:masked-field} with $\operatorname{Att}=\operatorname{Att}_{\mathrm C}$ gives
$\partial\hat{\mathbf v}_{\theta,i}/\partial\x_{t,j}=\mathbf0$ for $j>i$, which is exactly block lower triangularity.
\end{proof}

Causality thus constrains information flow without making generation autoregressive: all $T$ frames are initialized together and updated at every ODE step. The lower-triangular dependency prevents a suffix error or source perturbation from feeding back into an earlier prefix through subsequent solver evaluations. This structural guarantee does not by itself prove smaller numerical error or better samples, but it removes one route for global error propagation and preserves the forward accumulation semantics of incremental root motion. Bidirectional attention gives up prefix autonomy in exchange for full-sequence smoothing and two-sided propagation of spatial constraints, which is beneficial when coordinates are absolute and globally coupled. These complementary properties motivate the attention variants evaluated in~\Sref{sec:attention-study}.

\end{document}